\documentclass[11pt]{article}

\usepackage{fullpage}
\usepackage{hyperref}       
\usepackage{url}            
\usepackage{booktabs}       
\usepackage{amsfonts}       

\usepackage{amsmath,amsbsy,amsfonts,amssymb,amsthm,units}
\usepackage{graphicx}
\usepackage{algorithm, algorithmic}
\usepackage{color,cases}
\usepackage{array}
\newcolumntype{K}[1]{>{\centering\arraybackslash}p{#1}}

\theoremstyle{plain}
\newtheorem{theorem}{Theorem}[section]

\newtheorem{corollary}[theorem]{Corollary}
\newtheorem{assumption}[theorem]{Assumption}
\newtheorem{lemma}[theorem]{Lemma}
\newtheorem{proposition}[theorem]{Proposition}

\theoremstyle{definition}

\newtheorem{definition}[theorem]{Definition}

\newcommand{\vect}[1]{\ensuremath{\mathbf{#1}}}
\newcommand{\mat}[1]{\ensuremath{\mathbf{#1}}}

\newcommand{\grad}{\nabla}

\newcommand{\pgrad}[1]{ \frac{\partial}{\partial{#1} }}

\newcommand{\Unif}{\textrm{Unif}}
\newcommand{\abs}[1]{\left|{#1}\right|}
\newcommand{\norm}[1]{\left\|{#1}\right\|}
\newcommand{\fnorm}[1]{\|{#1}\|_{\text{F}}}
\newcommand{\infnorm}[1]{\|{#1}\|_\infty}

\newcommand{\sigmin}{\sigma_{\min}}
\newcommand{\tr}{\text{tr}}
\newcommand{\defeq}{\stackrel{\textrm{def}}{=}}

\newcommand{\rank}{\text{rank}}

\newcommand{\trans}{^{\top}}
\newcommand{\poly}{\text{poly}}

\newcommand{\proj}{\mathcal{P}}

\newcommand{\R}{\mathbb{R}}

\newcommand{\Rdd}{\mathbb{R}^{d\times d}}
\newcommand{\E}{\mathbb{E}}
\newcommand{\Var}{\text{Var}}

\newcommand{\A}{\mat{A}}
\newcommand{\B}{\mat{B}}
\newcommand{\C}{\mat{C}}
\renewcommand{\S}{\mat{S}}
\newcommand{\I}{\mat{I}}
\newcommand{\M}{\mat{M}}
\newcommand{\D}{\mat{D}}
\newcommand{\mP}{\mat{P}}
\newcommand{\mQ}{\mat{Q}}
\newcommand{\mR}{\mat{R}}
\newcommand{\U}{\mat{U}}
\newcommand{\V}{\mat{V}}
\newcommand{\W}{\mat{W}}
\newcommand{\X}{\mat{X}}
\newcommand{\Y}{\mat{Y}}
\newcommand{\mSigma}{\mat{\Sigma}}

\newcommand{\e}{\vect{e}}
\renewcommand{\u}{\vect{u}}
\renewcommand{\v}{\vect{v}}
\newcommand{\w}{\vect{w}}
\newcommand{\x}{\vect{x}}

\newcommand{\fE}{\mathfrak{E}}
\newcommand{\fF}{\mathfrak{F}}

\newcommand{\tU}{\tilde{\U}_t}
\newcommand{\tV}{\tilde{\V}_t}
\newcommand{\tX}{\tilde{\X}}
\newcommand{\tY}{\tilde{\Y}}
\newcommand{\tx}{\tilde{\x}}

\newcommand{\cn}{\kappa}
\newcommand{\nn}{\nonumber}
\newcommand{\order}[1]{O(#1)}

\definecolor{WildStrawberry}{RGB}{255,67,164}

\title{Provable Efficient Online Matrix Completion via Non-convex Stochastic Gradient Descent}

\author{Chi Jin\footnote{UC Berkeley. Email: chijin@cs.berkeley.edu} \and Sham M. Kakade\footnote{University of Washington. Email: sham@cs.washington.edu} \and Praneeth Netrapalli\footnote{Microsoft Research New England. Email: praneeth@microsoft.com}}

\begin{document}

\maketitle

\begin{abstract}
Matrix completion, where we wish to recover a low rank matrix by observing a few entries from it, is a widely studied problem in both theory and practice with wide applications. Most of the provable algorithms so far on this problem have been restricted to the offline setting where they provide an estimate of the unknown matrix using all observations simultaneously. However, in many applications, the online version, where we observe one entry at a time and dynamically update our estimate, is more appealing.
While existing algorithms are efficient for the offline setting, they could be highly inefficient for the online setting.

In this paper, we propose the first provable, efficient online algorithm for matrix completion. Our algorithm starts from an initial estimate of the matrix and then performs non-convex stochastic gradient descent (SGD). After every observation, it performs a fast update involving only one row of two tall matrices, giving near linear total runtime. Our algorithm can be naturally used in the offline setting as well, where it gives competitive sample complexity and runtime to state of the art algorithms. Our proofs introduce a general framework to show that SGD updates tend to stay away from saddle surfaces and could be of broader interests for other non-convex problems to prove tight rates.


\end{abstract}


\section{Introduction}
Low rank matrix completion refers to the problem of recovering a low rank matrix by observing the values of only a tiny fraction of its entries. This problem arises in several applications such as video denoising~\cite{ji2010robust}, phase retrieval~\cite{candes2015phase} and most famously in movie recommendation engines~\cite{Koren09}. In the context of recommendation engines for instance, the matrix we wish to recover would be user-item rating matrix where each row corresponds to a user and each column corresponds to an item. Each entry of the matrix is the rating given by a user to an item. Low rank assumption on the matrix is inspired by the intuition that rating of an item by a user depends on only a few hidden factors, which are much fewer than the number of users or items. The goal is to estimate the ratings of all items by users given only partial ratings of items by users, which would then be helpful in recommending new items to users.

The seminal works of Cand{\`e}s and Recht~\cite{CandesR2007} first identified regularity conditions under which low rank matrix completion can be solved in polynomial time using convex relaxation -- low rank matrix completion could be ill-posed and NP-hard in general without such regularity assumptions~\cite{HardtMRW14}. Since then, a number of works have studied various algorithms under different settings for matrix completion: weighted and noisy matrix completion, fast convex solvers, fast iterative non-convex solvers, parallel and distributed algorithms and so on.

Most of this work however deals only with the offline setting where all the observed entries are revealed at once and the recovery procedure does computation using all these observations simultaneously. However in several applications~\cite{davidson2010youtube,1167344}, we encounter the online setting where observations are only revealed sequentially and at each step the recovery algorithm is required to maintain an estimate of the low rank matrix based on the observations so far. Consider for instance recommendation engines, where the low rank matrix we are interested in is the user-item rating matrix. While we make an observation only when a user rates an item, at any point of time, we should have an estimate of the user-item rating matrix based on all prior observations so as to be able to continuously recommend items to users. Moreover, this estimate should get better as we observe more ratings.

Algorithms for offline matrix completion can be used to solve the online version by rerunning the algorithm after every additional observation. However, performing so much computation for every observation seems wasteful and is also impractical. For instance, using alternating minimization, which is among the fastest known algorithms for the offline problem, would mean that we take several passes of the entire data for every additional observation. This is simply not feasible in most settings. Another natural approach is to group observations into batches and do an update only once for each batch. This however induces a lag between observations and estimates which is undesirable. To the best of our knowledge, there is no known \emph{provable, efficient, online algorithm} for matrix completion.

On the other hand, in order to deal with the online matrix completion scenario in practical applications, several heuristics (with no convergence guarantees) have been proposed in literature~\cite{brand2003fast,luo2012incremental}. Most of these approaches are based on starting with an estimate of the matrix and doing fast updates of this estimate whenever a new observation is presented. One of the update procedures used in this context is that of stochastic gradient descent (SGD) applied to the following non-convex optimization problem
\begin{align}
	\min_{\U,\V} \fnorm{\M - \U \V\trans}^2 \quad \mbox{s.t.} \quad \U  \in \R^{d_1\times k},\V \in \R^{d_2\times k},\label{eqn:prob}
\end{align}
where $\M$ is the unknown matrix of size $d_1\times d_2$, $k$ is the rank of $\M$ and $\U\V\trans$ is a low rank factorization of $\M$ we wish to obtain. The algorithm starts with some $\U_0$ and $\V_0$, and given a new observation $(\M)_{ij}$, SGD updates the $i^{\textrm{th}}$-row and the $j^{\textrm{th}}$-row of the current iterates $\U_t$ and $\V_t$ respectively by
\begin{align}
	\U_{t+1}^{(i)} &= \U_{t}^{(i)} - 2 \eta d_1 d_2 \left(\U_t\V_t\trans - \M\right)_{ij} \V_{t}^{(j)}, \mbox{ and}, \nonumber \\
	\V_{t+1}^{(j)} &= \V_{t}^{(j)} - 2 \eta d_1 d_2 \left(\U_t\V_t\trans - \M\right)_{ij} \U_{t}^{(i)}, \label{eqn:SGD}
\end{align}
where $\eta$ is an appropriately chosen stepsize, and $\U^{(i)}$ denote the $i^{\textrm{th}}$ row of matrix $\U$. Note that each update modifies only one row of the factor matrices $\U$ and $\V$, and the computation only involves one row of $\U, \V$ and the new observed entry $(\M)_{ij}$ and hence are extremely fast. These fast updates make SGD extremely appealing in practice. Moreover, SGD, in the context of matrix completion, is also useful for parallelization and distributed implementation~\cite{RechtR2013}.


\subsection{Our Contributions}  In this work we present the first provable efficient algorithm for online matrix completion by showing that SGD~\eqref{eqn:SGD} with a good initialization 
converges to a true factorization of $\M$ at a geometric rate. Our main contributions are as follows. 
\begin{itemize}
	\item	We provide the \emph{first provable, efficient, online} algorithm for matrix completion. Starting with a good initialization, after each observation, the algorithm makes quick updates each taking time $\order{k^3}$ and requires $\order{\mu d k \cn^4 (k+\log \frac{\fnorm{\M}}{\epsilon})\log d}$ observations to reach $\epsilon$ accuracy, where $\mu$ is the incoherence parameter, $d = \max(d_1,d_2)$, $k$ is the rank and $\cn$ is the condition number of $\M$.
	\item	Moreover, our result features both sample complexity and total runtime \emph{linear in $d$}, and is \emph{competitive} to even the best existing offline results for matrix completion.
	(either improve over or is incomparable, i.e., better in some parameters and worse in others, to these results).
	See Table~\ref{tab:comp} for the comparison.
	\item	To obtain our results, we introduce \emph{a general framework to show SGD updates tend to stay away from saddle surfaces}. In order to do so, we consider distances from saddle surfaces, show that they behave like sub-martingales under SGD updates and use martingale convergence techniques to conclude that the iterates stay away from saddle surfaces. While~\cite{Sun2015guaranteed} shows that SGD updates stay away from saddle surfaces, the stepsizes they can handle are quite small 
	(scaling as $1/\poly(d_1,d_2)$), leading to suboptimal computational complexity. Our framework makes it possible to establish the same statement for much larger step sizes, giving us near-optimal runtime. We believe these techniques may be applicable in other non-convex settings as well.
\end{itemize}


\subsection{Related Work}\label{sec:related}
In this section we will mention some more related work.

\textbf{Offline matrix completion}: There has been a lot of work on designing offline algorithms for matrix completion, we provide the detailed comparison with our algorithm in Table \ref{tab:comp}. The nuclear norm relaxation algorithm 
\cite{Recht2009} 
has near-optimal sample complexity for this problem but is computationally expensive. Motivated by the empirical success of non-convex heuristics, a long line of works, 
\cite{keshavan2012efficient,Hardt2014understanding,jain2014fast,Sun2015guaranteed} 
and so on, has obtained convergence guarantees for alternating minimization, gradient descent, projected gradient descent etc. Even the best of these are suboptimal in sample complexity by $\poly(k,\cn)$ factors. Our sample complexity is better than that of~\cite{keshavan2012efficient} and is incomparable to those of~\cite{Hardt2014understanding,jain2014fast}.
To the best of our knowledge, the only provable online algorithm for this problem is that of Sun and Luo \cite{Sun2015guaranteed}.  However the stepsizes they suggest are quite small, leading to suboptimal computational complexity by factors of $\poly(d_1,d_2)$.
The runtime of our algorithm is linear in $d$, which makes $\poly(d)$ improvements over it.



\textbf{Other models for online matrix completion}:
Another variant of online matrix completion studied in the literature is where observations are made on a column by column basis e.g.,~\cite{krishnamurthy2013low,yun2015streaming}. These models can give improved offline performance in terms of space and could potentially work under relaxed regularity conditions. However, they do not tackle the version where only entries (as opposed to columns) are observed.

\textbf{Non-convex optimization}: 
Over the last few years, there has also been a significant amount of work in designing other efficient algorithms for solving non-convex problems. Examples include eigenvector computation~\cite{de2014global,jain2016matching}, sparse coding~\cite{mairal2010online,arora2015simple} etc.
For general non-convex optimization, an interesting line of recent work is that of
\cite{ge2015escaping}, which proves gradient descent with noise can also escape saddle point, but they only provide polynomial rate without explicit dependence. Later
\cite{lee2016gradient,panageas2016gradient} show that without noise, the space of points from where gradient descent converges to a saddle point is a measure zero set. However, they do not provide a rate of convergence. 
Another related piece of work to ours is~\cite{jain2015computing}, proves global convergence along with rates of convergence, for the special case of computing matrix squareroot.
During the preparation of this draft, the recent work~\cite{Ge2016matrix} was announced which proves the global convergence of SGD for matrix completion and can also be applied to the online setting. However, their result only deals with the case where $\M$ is positive semidefinite (PSD) and their rate is still suboptimal by factors of $\poly(d_1,d_2)$.

\begin{table}[t]
	\begin{center}
		{\renewcommand{\arraystretch}{1.3}
		\begin{tabular}{  >{\centering\arraybackslash}m{1.4in} >{\centering\arraybackslash}m{1.4in} >{\centering\arraybackslash}m{1.4in} >{\centering\arraybackslash}m{0.6in} }
			\hline
			\textbf{Algorithm} & \textbf{Sample complexity} & \textbf{Total runtime} & \textbf{Online?} \\ 
			\hline
			Nuclear Norm \cite{Recht2009} & $\widetilde{O}(\mu d k)$ & $\widetilde{O}(d^3/\sqrt{\epsilon})$ & No \\ 
			Alternating minimization \cite{keshavan2012efficient} & $\widetilde{O}(\mu d k \cn^8 \log \frac{1}{\epsilon})$ & $\widetilde{O}(\mu d k^2 \cn^8 \log \frac{1}{\epsilon})$
			& No \\ 
			Alternating minimization \cite{Hardt2014understanding} & $\widetilde{O}\left(\mu d k^2 \cn^2 \left(k + \log \frac{1}{\epsilon}\right)\right)$ & $\widetilde{O}\left(\mu d k^3 \cn^2 \left(k + \log \frac{1}{\epsilon}\right)\right)$
			& No \\ 
			Projected gradient descent\cite{jain2014fast} & $\widetilde{O}(\mu d k^5)$ & $\widetilde{O}(\mu d k^7 \log \frac{1}{\epsilon})$
			& No \\ 
			SGD \cite{Sun2015guaranteed} & $\widetilde{O}(\mu^2d k^7 \cn^6)$ & $\poly(\mu, d, k, \cn)\log\frac{1}{\epsilon}$ & Yes \\
			SGD \cite{Ge2016matrix}\footnotemark  & $d\cdot\poly(\mu, k, \cn)$ & $\poly(\mu, d, k, \cn, \frac{1}{\epsilon})$ & Yes \\
			\textbf{Our result} & $\widetilde{O}\left(\mu d k \cn^4 \left(k+\log \frac{1}{\epsilon}\right)\right)$ & $\widetilde{O}\left(\mu d k^4 \cn^4 \log \frac{1}{\epsilon}\right)$ & Yes \\
		\hline
		\end{tabular}
		}
		\caption{Comparison of sample complexity and runtime of our algorithm with existing  algorithms in order to obtain Frobenius norm error $\epsilon$. $\widetilde{O}(\cdot)$ hides $\log d$ factors. 
		See Section~\ref{sec:related} for more discussion.
		}
		\label{tab:comp}
	\end{center}
\end{table}
\footnotetext{This result only applies to the case where $\M$ is symmetric PSD}


\subsection{Outline}
The rest of the paper is organized as follows. In Section~\ref{sec:prelims} we formally describe the problem and all relevant parameters. In Section~\ref{sec:result}, we present our algorithms, results and some of the key intuition behind our results. In Section~\ref{sec:sketch} we give proof outline for our main results. We conclude in Section~\ref{sec:conclu}.
All formal proofs are deferred to the Appendix.

\section{Preliminaries}\label{sec:prelims}

In this section, we introduce our notation, formally define the matrix completion problem and regularity assumptions that make the problem tractable.

\subsection{Notation}
We use $[d]$ to denote $\{1, 2, \cdots, d\}$. We use bold capital letters $\A, \B$ to denote matrices and bold lowercase letters $\u, \v$ to denote vectors. $\A_{ij}$ means the $(i,j)^{\textrm{th}}$ entry of matrix $\A$. $\norm{\w}$ denotes the $\ell_2$-norm of vector $\w$ and $\norm{\A}$/$\fnorm{\A}$/$\infnorm{\A}$ denotes the spectral/Frobenius/infinity norm of matrix $\A$. $\sigma_i(\A)$ denotes the $i^{\textrm{th}}$ largest singular value of $\A$ and $\sigmin(\A)$ denotes the smallest singular value of $\A$. We also let $\cn(\A) = \norm{\A}/\sigmin(\A)$ denote the condition number of $\A$ (i.e., the ratio of largest to smallest singular value). Finally, for orthonormal bases of a subspace $\W$, we also use $\proj_\W = \W \W\trans$ to denote the projection to the subspace spanned by $\W$.

\subsection{Problem statement and assumptions}
Consider a general rank $k$ matrix $\M \in \R^{d_1 \times d_2}$. Let $\Omega \subset [d_1]\times[d_2]$ be a subset of coordinates, which are sampled uniformly and independently from $[d_1]\times[d_2]$. We denote
$\proj_\Omega(\M)$ to be the projection of $\M$ on set $\Omega$ so that:
\begin{equation*}
[\proj_\Omega(\M)]_{ij}
= \left\{\begin{array}{l l}
\M_{ij}, &\text{if~} (i,j) \in \Omega \\
0, & \text{if~} (i,j) \not \in \Omega
\end{array}\right.
\end{equation*}
Low rank matrix completion is the task of recovering $\M$ by only observing $\proj_\Omega(\M)$. This task is ill-posed and NP-hard in general~\cite{HardtMRW14}. In order to make this tractable, we make by now standard assumptions about the structure of $\M$.
\begin{definition}
Let $\W \in \R^{d\times k}$ be an orthonormal basis of a subspace of $\R^d$ of dimension $k$. The coherence of $\W$ is defined to be
\begin{equation*}
\mu(\W) \defeq \frac{d}{k} \max_{1\le i \le d} \norm{\proj_\W \e_i}^2
= \frac{d}{k} \max_{1\le i \le d}\norm{ \e_i\trans \W}^2
\end{equation*}
\end{definition}
\begin{assumption}[$\mu$-incoherence\cite{CandesR2007,Recht2009}]
We assume $\M$ is $\mu$-incoherent, i.e., 
$\max\{\mu(\X), \mu(\Y)\} \le \mu$, where $\X \in \R^{d_1\times k}, \Y \in \R^{d_2\times k}$ are the left and right singular vectors of $\M$.
\end{assumption}

%
%


\section{Main Results}\label{sec:result}
In this section, we present our main result. We will first state result for a special case where $\M$ is a symmetric positive semi-definite (PSD) matrix, where the algorithm and analysis are much simpler. We will then discuss the general case.

\subsection{Symmetric PSD Case}
Consider the special case where $\M$ is symmetric PSD. We let $d \defeq d_1 = d_2$, and we can parametrize a rank $k$ symmetric PSD matrix by $\U\U\trans$ where $\U \in \R^{d \times k}$. Our algorithm for this case is given in Algorithm \ref{algo:MC_Sym}. The following theorem provides guarantees on the performance of Algorithm~\ref{algo:MC_Sym}. The algorithm starts by using an initial set of samples $\Omega_{\textrm{init}}$ to construct a crude approximation to the low rank of factorization of $\M$. It then observes samples from $\M$ one at a time and updates its factorization after every observation. Note that each update step modifies two rows of $\U_t$ and hence takes time $\order{k}$.
\begin{algorithm}[!ht]
\caption{Online Algorithm for PSD Matrix Completion.}\label{algo:MC_Sym}
\begin{algorithmic}
\renewcommand{\algorithmicrequire}{\textbf{Input: }}
\renewcommand{\algorithmicensure}{\textbf{Output: }}
\REQUIRE Initial set of uniformly random samples $\Omega_{\textrm{init}}$ of a symmetric PSD matrix $\M \in \Rdd$,
learning rate $\eta$, iterations $T$
\ENSURE  $\U$ such that $\U\U\trans \approx \M$
\STATE $\U_0\U_0\trans \leftarrow $ top $k$ SVD of $\frac{d^2}{\abs{\Omega_{\textrm{init}}}}\proj_{\Omega_{\textrm{init}}}(\M)$
\FOR{$t = 0,\cdots,T-1$}
\STATE Observe $\M_{ij}$ where $(i,j) \sim \Unif\left([d]\times [d]\right)$
\STATE $\U_{t+1} \leftarrow \U_t - 2\eta d^2(\U_t\U_t\trans - \M)_{ij} (\e_i\e_j\trans + \e_j\e_i\trans)\U_t$
\ENDFOR
\STATE \textbf{Return} $\U_T$
\end{algorithmic}
\end{algorithm}
%

\begin{theorem} \label{thm:main_Sym}
	Let $\M \in \R^{d\times d}$ be a rank $k$, symmetric PSD matrix with $\mu$-incoherence. There exist some absolute constants $c_0$ and $c$ such that if $\abs{\Omega_{\textrm{init}}} \geq c_0\mu dk^2 \kappa^2(\M) \log d $, learning rate $\eta \leq \frac{c}{\mu dk \cn^3(\M)\norm{\M} \log d}$, then for any fixed $T \geq 1$, with probability at least $1-\frac{T}{d^{10}}$, we will have for all $t\le T$ that:
	\begin{equation*}
	\fnorm{\U_t\U_t\trans - \M }^2 \le \left(1-\frac{1}{2}\eta\cdot\sigmin(\M)\right)^t\left(\frac{1}{10}\sigmin(\M)\right)^2.
	\end{equation*}
\end{theorem}
\textbf{Remarks}:
\begin{itemize}
	\item	The algorithm uses an initial set of observations $\Omega_{\textrm{init}}$ to produce a warm start iterate $\U_0$, then enters the online stage, where it performs SGD.
	\item	The sample complexity of the warm start phase is $\order{\mu dk^2 \kappa^2(\M) \log d }$. The initialization consists of a top-$k$ SVD on a sparse matrix, whose runtime is $\order{\mu dk^3 \kappa^2(\M) \log d }$.
	\item	For the online phase (SGD), if we choose $\eta = \frac{c}{\mu dk \cn^3(\M)\norm{\M} \log d}$, the number of observations $T$ required for the error $\fnorm{\U_T \U_T \trans - \M}$ to be smaller than $\epsilon$ is $\order{\mu d k \cn(\M)^4 \log d \log \frac{\sigmin(\M)}{\epsilon}}$.
	\item	Since each SGD step modifies two rows of $\U_t$, its runtime is $\order{k}$ with a total runtime for online phase of $\order{kT}$.
\end{itemize}
Our proof approach is to essentially show that the objective function is well-behaved (i.e., is smooth and strongly convex) in a local neighborhood of the warm start region, and then use standard techniques to show that SGD obtains geometric convergence in this setting. The most challenging and novel part of our analysis comprises of showing that the iterate does not leave this local neighborhood while performing SGD updates. Refer Section~\ref{sec:sketch} for more details on the proof outline.


\subsection{General Case} \label{sec:Asym_main}
Let us now consider the general case where $\M \in \R^{d_1\times d_2}$ can be factorized as $\U\V\trans$ with $\U\in \R^{d_1\times k}$ and $\V \in \R^{d_2 \times k}$. In this scenario, we denote $d = \max \{d_1, d_2\}$. We recall our remarks from the previous section that our analysis of the performance of SGD depends on the smoothness and strong convexity properties of the objective function in a local neighborhood of the iterates. Having $\U \neq \V$ introduces additional challenges in this approach since for any nonsingular $k$-by-$k$ matrix $\C$, and $\U' \defeq \U \C\trans, \V' \defeq \V \C^{-1}$, we have $\U'\V'{}\trans = \U \V \trans$. Suppose for instance $\C$ is a very small scalar times the identity i.e., $\C = \delta \I$ for some small $\delta > 0$. In this case, $\U'$ will be large while $\V'$ will be small. This drastically deteriorates the smoothness and strong convexity properties of the objective function in a neighborhood of $(\U',\V')$.

\begin{algorithm}[!ht]
\caption{Online Algorithm for Matrix Completion (Theoretical)}\label{algo:MC_Asym}
\begin{algorithmic}
\renewcommand{\algorithmicrequire}{\textbf{Input: }}
\renewcommand{\algorithmicensure}{\textbf{Output: }}
\REQUIRE Initial set of uniformly random samples $\Omega_{\textrm{init}}$ of $\M \in \R^{d_1\times d_2}$,
learning rate $\eta$, iterations $T$
\ENSURE  $\U, \V$ such that $\U\V\trans \approx \M$
\STATE $\U_0\V_0\trans \leftarrow $ top $k$ SVD of $\frac{d_1 d_2}{\abs{\Omega_{\textrm{init}}}}\proj_{\Omega_{\textrm{init}}}(\M)$
\FOR{$t = 0,\cdots,T-1$}
\STATE $\W_U \D \W_V \trans \leftarrow \text{SVD}(\U_{t}\V_{t}\trans)$
\STATE $\tU \leftarrow \W_U \D^{\frac{1}{2}}, \quad \tV \leftarrow \W_V \D^{\frac{1}{2}}$
\STATE Observe $\M_{ij}$ where $(i,j) \sim \Unif\left([d]\times [d]\right)$
\STATE $\U_{t+1} \leftarrow \tU - 2\eta d_1 d_2(\tU\tV\trans - \M)_{ij} \vect{e}_i\vect{e}_j\trans\tV$
\STATE $\V_{t+1} \leftarrow \tV - 2\eta d_1 d_2(\tU\tV\trans - \M)_{ij} \vect{e}_j\vect{e}_i\trans\tU
$
\ENDFOR
\STATE \textbf{Return} $\U_T, \V_T$.
\end{algorithmic}
\end{algorithm}
To preclude such a scenario, we would ideally like to renormalize after each step by doing $\tU \leftarrow \W_U \D^{\frac{1}{2}}, \tV \leftarrow \W_V \D^{\frac{1}{2}}$, where $\W_U \D \W_V \trans$ is the SVD of matrix $\U_{t}\V_{t}\trans$. This algorithm is described in Algorithm~\ref{algo:MC_Asym}. However, a naive implementation of Algorithm~\ref{algo:MC_Asym}, especially the SVD step, would incur $O(\min\{d_1, d_2\})$ computation per iteration, resulting in a runtime overhead of $\order{d}$ over both the online PSD case (i.e., Algorithm~\ref{algo:MC_Sym}) as well as the near linear time offline algorithms (see Table~\ref{tab:comp}). It turns out that we can take advantage of the fact that in each iteration we only update a single row of $\U_t$ and a single row of $\V_t$, and do efficient (but more complicated) update steps instead of doing an SVD on $d_1\times d_2$ matrix. The resulting algorithm is given in Algorithm \ref{algo:MC_Asym_Prac}. The key idea is that in order to implement the updates, it suffices to do an SVD of $\U_t\trans \U_t$ and $\V_t \trans \V_t$ which are $k \times k$ matrices. So the runtime of each iteration is at most $O(k^3)$. The following lemma shows the equivalence between Algorithms~\ref{algo:MC_Asym} and~\ref{algo:MC_Asym_Prac}.
\begin{algorithm}[!ht]
\caption{Online Algorithm for Matrix Completion (Practical)}\label{algo:MC_Asym_Prac}
\begin{algorithmic}
\renewcommand{\algorithmicrequire}{\textbf{Input: }}
\renewcommand{\algorithmicensure}{\textbf{Output: }}
\REQUIRE Initial set of uniformly random samples $\Omega_{\textrm{init}}$ of $\M \in \R^{d_1\times d_2}$,
learning rate $\eta$, iterations $T$
\ENSURE  $\U, \V$ such that $\U\V\trans \approx \M$
\STATE $\U_0\V_0\trans \leftarrow $ top $k$ SVD of $\frac{d_1 d_2}{\Omega_{\textrm{init}}}\proj_{\Omega_{\textrm{init}}}(\M)$
\FOR{$t = 0,\cdots,T-1$}
\STATE $\mR_U \D_U \mR_U \trans \leftarrow \text{SVD}(\U_{t}\trans\U_{t})$
\STATE $\mR_V \D_V \mR_V \trans \leftarrow \text{SVD}(\V_{t}\trans\V_{t})$
\STATE $\mQ_U \D \mQ_V \trans \leftarrow \text{SVD}(\D_U^\frac{1}{2} \mR_U\trans \mR_V (\D_V^{\frac{1}{2}})\trans)$
\STATE Observe $\M_{ij}$ where $(i,j) \sim \Unif\left([d]\times [d]\right)$
\STATE $\U_{t+1} \leftarrow \U_t - 2\eta d_1 d_2(\U_{t}\V_{t}\trans - \M)_{ij} \vect{e}_i\vect{e}_j\trans\V_{t}\mR_V\D_V^{-\frac{1}{2}} \mQ_V \mQ_U\trans \D_U^{\frac{1}{2}} \mR_U\trans$
\STATE $\V_{t+1} \leftarrow \V_t - 2\eta d_1 d_2(\U_{t}\V_{t}\trans - \M)_{ij} \vect{e}_j\vect{e}_i\trans\U_{t}
\mR_U\D_U^{-\frac{1}{2}} \mQ_U \mQ_V\trans \D_V^{\frac{1}{2}} \mR_V\trans$
\ENDFOR
\STATE \textbf{Return} $\U_T, \V_T$.
\end{algorithmic}
\end{algorithm}

\begin{lemma} \label{lem:equiv}
	Algorithm \ref{algo:MC_Asym} and Algorithm \ref{algo:MC_Asym_Prac} are equivalent in the sense that: given same observations from $\M$ and other inputs, the outputs of Algorithm~\ref{algo:MC_Asym}, $\U, \V$ and those of Algorithm~\ref{algo:MC_Asym_Prac}, $\U', \V'$ satisfy $\U\V\trans = \U'\V'{}\trans$.
\end{lemma}

Since the output of both algorithms is the same, we can analyze Algorithm~\ref{algo:MC_Asym} (which is easier than that of Algorithm~\ref{algo:MC_Asym_Prac}), while implementing Algorithm~\ref{algo:MC_Asym_Prac} in practice.
The following theorem is the main result of our paper which presents guarantees on the performance of Algorithm~\ref{algo:MC_Asym}.
%
\begin{theorem} \label{thm:main_Asym}
	Let $\M \in \R^{d_1\times d_2}$ be a rank $k$ matrix with $\mu$-incoherence and let $d \defeq \max(d_1,d_2)$. There exist some absolute constants $c_0$ and $c$ such that if $\abs{\Omega_{\textrm{init}}} \geq c_0\mu dk^2 \kappa^2(\M) \log d$, learning rate $\eta \leq \frac{c}{\mu dk \cn^3(\M)\norm{\M} \log d}$, then for any fixed $T \geq 1$, with probability at least $1-\frac{T}{d^{10}}$, we will have for all $t\le T$ that:
	\begin{equation*}
	\fnorm{\U_t\V_t\trans - \M }^2 \le \left(1-\frac{1}{2}\eta\cdot\sigmin(\M)\right)^t\left(\frac{1}{10}\sigmin(\M)\right)^2.
	\end{equation*}
\end{theorem}
\textbf{Remarks}:
\begin{itemize}
	\item	Just as in the case of PSD matrix completion (Theorem~\ref{thm:main_Sym}), Algorithm~\ref{algo:MC_Asym} needs a an initial set of observations $\Omega_{\textrm{init}}$ to provide a warm start $\U_0$ and $\V_0$ after which it performs SGD.
	\item	The sample complexity and runtime of the warm start phase are the same as in symmetric PSD case. The stepsize $\eta$ and the number of observations $T$ to achieve $\epsilon$ error in online phase (SGD) are also the same as in symmetric PSD case.
	\item	However, runtime of each update step in online phase is $\order{k^3}$ with total runtime for online phase $\order{k^3 T}$.
\end{itemize} 
The proof of this theorem again follows a similar line of reasoning as that of Theorem~\ref{thm:main_Sym} by first showing that the local neighborhood of warm start iterate has good smoothness and strong convexity properties and then use them to show geometric convergence of SGD. Proof of the fact that iterates do not move away from this local neighborhood however is significantly more challenging due to renormalization steps in the algorithm. Please see Appendix~\ref{sec:proof_Asym} for the full proof.

\section{Proof Sketch}\label{sec:sketch}

In this section we will provide the intuition and proof sketch for our main results. For simplicity and highlighting the most essential ideas, we will mostly focus on the symmetric PSD case (Theorem \ref{thm:main_Sym}). For the asymmetric case, though the high-level ideas are still valid, a lot of additional effort is required to address the renormalization step in Algorithm \ref{algo:MC_Asym}. This makes the proof more involved.

First, note that our algorithm for the PSD case consists of an initialization and then stochastic descent steps. The following lemma provides guarantees on the error achieved by the initial iterate $\U_0$.
\begin{lemma} \label{lem:initial_Sym}
Let $\M \in \R^{d\times d}$ be a rank-$k$ PSD matrix with $\mu$-incoherence. There exists a constant $c_0$ such that if $\abs{\Omega_{\textrm{init}}} \geq c_0\mu dk^2 \kappa^2(\M) \log d$, then with probability at least $1-\frac{1}{d^{10}}$, the top-$k$ SVD of $\frac{d^2}{\abs{\textrm{init}}}\proj_{\Omega_{\textrm{init}}}(\M)$ satisfies
Then there exists universal constant $c_0$, for any
$m \ge $,  we have:
\begin{equation}\label{initial_cond_Sym}
\fnorm{\M - \U_0\U_0\trans} \le \frac{1}{20}\sigmin(\M) \quad \text{~and~} \quad \max_j \norm{\e_j\trans \U_0}^2 \le \frac{ 10\mu k \cn(\M)}{d}\norm{\M}
\end{equation}
\end{lemma}

By Lemma~\ref{lem:initial_Sym}, we know the initialization algorithm already gives $\U_0$ in the local region given by Eq.\eqref{initial_cond_Sym}. Intuitively, stochastic descent steps should keep doing local search within this local region.

To establish linear convergence on $\fnorm{\U_t\U_t\trans - \M}^2$ and obtain final result, we first establish several important lemmas describing the properties of this local regions. Throughout this section, we always denote $\text{SVD}(\M) = \X\S\X\trans$, where $\X\in \R^{d\times k}$, and diagnal matrix $\S\in \R^{k\times k}$. We postpone all the formal proofs in Appendix.

\begin{lemma} \label{lem:smoothness_sketch}
For function $f(\U) = \fnorm{\M-\U\U\trans}^2$ and any $\U_1, \U_2 \in \{\U | \norm{\U} \le \Gamma\}$, we have:
\begin{equation*}
\fnorm{\nabla f(\U_1) - \nabla f(\U_2)} \le 16 \max\{\Gamma^2, \norm{\M}\} \cdot \fnorm{\U_1 - \U_2}
\end{equation*}
\end{lemma}

\begin{lemma} \label{lem:pseudostronglyconvex_sketch}
For function $f(\U) = \fnorm{\M-\U\U\trans}^2$ and any $\U \in \{\U | \sigmin(\X\trans\U) \ge \gamma \}$, we have:
\begin{equation*}
\|\nabla f(\U)\|^2_F \ge 4 \gamma^2 f(\U)
\end{equation*}
\end{lemma}

Lemma \ref{lem:smoothness_sketch} tells function $f$ is smooth if spectral norm of $\U$ is not very large. On the other hand, $\sigmin(\X\trans\U)$ not too small requires both $\sigmin(\U\trans\U)$ and $\sigmin(\X\trans\W)$ are not too small, where $\W$ is top-k eigenspace of $\U\U\trans$. That is,
Lemma \ref{lem:pseudostronglyconvex_sketch} tells function $f$ has a property similar to strongly convex in standard optimization literature, if $\U$ is rank k in a robust sense ($\sigma_k(\U)$ is not too small), and the angle between the top k eigenspace of $\U\U\trans$ and the top k eigenspace $\M$ is not large.

\begin{lemma} \label{lem:localguarantee_sketch}
Within the region $\mathcal{D} = \{\U | \norm{\M - \U\U\trans}_F \le \frac{1}{10}\sigma_{k}(\M)\}$, 
we have:
\begin{equation*}
\norm{\U} \le \sqrt{2\norm{\M}}, \quad\quad \sigmin(\X\trans \U) \ge \sqrt{\sigma_k(\M)/2}
\end{equation*}
\end{lemma}

Lemma \ref{lem:localguarantee_sketch} tells inside region $\{\U | \norm{\M - \U\U\trans}_F \le \frac{1}{10}\sigma_{k}(\M)\}$, matrix $\U$ always has a good spectral property which gives preconditions for both Lemma \ref{lem:smoothness_sketch} and \ref{lem:pseudostronglyconvex_sketch}, where $f(\U)$ is both smooth and has a property very similar to strongly convex.

With above three lemmas, we already been able to see the intuition behind linear convergence in Theorem \ref{thm:main_Sym}. Denote stochastic gradient 
\begin{equation}\label{SG_main}
SG(\U) = 2 d^2(\U\U\trans - \M)_{ij} (\e_i\e_j\trans + \e_j\e_i\trans)\U
\end{equation}
where $SG(\U)$ is a random matrix depends on the randomness of sample $(i, j)$ of matrix $\M$. Then, the stochastic update step in Algorithm \ref{algo:MC_Sym} can be rewritten as:
\begin{equation*} 
\U_{t+1} \leftarrow \U_t - \eta SG(\U_t)
\end{equation*}
Let $f(\U)=\fnorm{\M - \U\U\trans}^2$, By easy caculation, we know $\E SG(\U) = \grad f(\U)$, that is $SG(\U)$ is unbiased. 
Combine Lemma \ref{lem:localguarantee_sketch} with Lemma \ref{lem:smoothness_sketch} and Lemma \ref{lem:pseudostronglyconvex_sketch}, we know within region $\mathcal{D}$ specified by Lemma \ref{lem:localguarantee_sketch}, 
we have function $f(\U)$ is $32\norm{\M}$-smooth, and $\|\nabla f(\U)\|^2_F \ge 2 \sigma_{\min}(\M) f(\U)$.

Let's suppose ideally, we always have $\U_0, \ldots, \U_t$ inside region $\mathcal{D}$, this directly gives:
\begin{align*}
\E f(\U_{t+1}) &\le \E f(\U_t) -\eta \E \langle\grad f(\U_t), SG(\U_t)\rangle
+ 16\eta^2 \norm{\M} \cdot \E \fnorm{SG(\U_t)}^2 \\
&= \E f(\U_t) -\eta \E \fnorm{\grad f(\U_t)}^2
+ 16\eta^2 \norm{\M} \cdot \E \fnorm{SG(\U_t)}^2 \\
&\le (1-2\eta\sigma_{\min}(\M)) \E f(\U_t) + 16\eta^2 \norm{\M} \cdot \E \fnorm{SG(\U_t)}^2 
\end{align*}
One interesting aspect of our main result is that we actually show linear convergence under the presence of noise in gradient. This is true because for the second-order ($\eta^2$) term above, we can roughly see from Eq.\eqref{SG_main} that $\fnorm{SG(\U)}^2 \le h(\U) \cdot f(\U)$, where $h(\U)$ is a factor depends on $\U$ and always bounded. That is, $SG(\U)$ enjoys self-bounded property --- $\fnorm{SG(\U)}^2$ will goes to zero, as objective function $f(\U)$ goes to zero. Therefore, by choosing learning rate $\eta$ appropriately small, we can have the first-order term always dominate the second-order term, which establish the linear convergence.

Now, the only remaining issue is to prove that ``$\U_0, \ldots, \U_t$ always stay inside local region $\mathcal{D}$''. In reality, we can only prove this statement with high probability due to the stochastic nature of the update.
This is also the most challenging part in our proof, which makes our analysis different from standard convex analysis, and uniquely required due to non-convex setting.

Our key theorem is presented as follows:
\begin{theorem} \label{thm:localcontrol_sketch}
Let $f(\U) = \norm{\U\U\trans - \M}_F^2$ and $g_i(\U) = \norm{\e_i\trans \U}^2$. 
Suppose initial $\U_0$ satisfying:
\begin{equation*}
f(\U_0) \le \left(\frac{\sigma_{\min}(\M)}{20 }\right)^2, \quad\quad \max_i g_i(\U_0) \le \frac{10\mu k \cn(\M)^2}{d}\norm{\M}
\end{equation*}
Then, there exist some absolute constant $c$ such that for any learning rate $\eta < \frac{c}{\mu dk \cn^3(\M)\norm{\M} \log d}$,
with at least $1-\frac{T}{d^{10}}$ probability, we will have for all $t\le T$ that:
\begin{equation} \label{cond_localcontrol}
f(\U_t) \le (1-\frac{1}{2}\eta\sigma_{\min}(\M))^t\left(\frac{\sigma_{\min}(\M)}{10}\right)^2, \quad\quad \max_i g_i(\U_t) \le \frac{20\mu k \cn(\M)^2}{d}\norm{\M}
\end{equation}

\end{theorem}

Note function $\max_i g_i(\U)$ indicates the incoherence of matrix $\U$.
Theorem \ref{thm:localcontrol_sketch} guarantees if inital $\U_0$ is in the local region which is incoherent and $\U_0\U_0\trans$ is close to $\M$, then with high probability for all steps $t\le T$, $\U_t$, $\U_t$ will always stay in a slightly relaxed local region, and $f(\U_t)$ has linear convergence.

It is not hard to show that all saddle point of $f(\U)$ satisfies $\sigma_{k}(\U) = 0$, and all local minima are global minima. Since $\U_0, \ldots, \U_t$ automatically stay in region $f(\U) \le (\frac{\sigma_{\min}(\M)}{10})^2$ with high probability, we know $\U_t$ also stay away from all saddle points.
The claim that $\U_0, \ldots, \U_t$ stays incoherent is essential to better control the variance and probability 1 bound of $SG(\U_t)$, so that we can have large step size and tight convergence rate.

The major challenging in proving Theorem \ref{thm:localcontrol_sketch} is to both prove $\U_t$ stays in the local region, and achieve good sample complexity and running time (linear in $d$) in the same time. This also requires the learning rate $\eta$ in Algorithm \ref{algo:MC_Sym} to be relatively large. Let the event $\fE_t$ denote the good event where $\U_0, \ldots, \U_t$ satisfies Eq.\eqref{cond_localcontrol}. Theorem \ref{thm:localcontrol_sketch} is claiming that $P(\fE_T)$ is large. The essential steps in the proof is contructing two supermartingles related to $f(\U_t)1_{\fE_t}$ and $g_i(\U_t)1_{\fE_t}$ (where $1_{(\cdot)}$ denote indicator function), and use Bernstein inequalty to show the concentration of supermartingales. The $1_{\fE_t}$term allow us the claim all previous $\U_0, \ldots, \U_t$ have all desired properties inside local region. 

Finally, we see Theorem \ref{thm:main_Sym} as a immediate corollary of Theorem \ref{thm:localcontrol_sketch}.



\section{Conclusion} \label{sec:conclu}
In this paper, we presented the first provable, efficient online algorithm for matrix completion, based on nonconvex SGD. In addition to the online setting, our results are also competitive with state of the art results in the offline setting. We obtain our results by introducing a general framework that helps us show how SGD updates self-regulate to stay away from saddle points. We hope our paper and results help generate interest in online matrix completion, and our techniques and framework prompt tighter analysis for other nonconvex  problems.
\bibliographystyle{plain}
\bibliography{matcomp}
\newpage
\appendix

\section{Proof of Initialization}

In this section, we will prove Lemma \ref{lem:initial_Sym} and 
a corresponding lemma for asymmetric case as follows (which will be used to prove Theorem \ref{thm:main_Asym}):
\begin{lemma} \label{lem:initial_Asym}
Assume $\M \in \R^{d_1\times d_2}$ is a rank $k$ matrix with $\mu$-incoherence, and $\Omega$ is a subset unformly i.i.d sampled from all coordinate.
Let $\U_0\V_0\trans$ be the top-$k$ SVD of $\frac{d_1 d_2}{m}\proj_{\Omega}(\M)$, where $|\Omega| =m$.
Let $d = \max\{d_1, d_2\}$.
Then there exists universal constant $c_0$, for any
$m \ge c_0\mu dk^2 \kappa^2(\M) \log d$, with probability at least $1-\frac{1}{d^{10}}$, we have:
\begin{align}
&\fnorm{\M - \U_0\V_0\trans} \le \frac{1}{20}\sigmin(\M), \nn \\
&\max_i \norm{\e_i\trans \U_0\V_0\trans}^2 \le \frac{ 10\mu k }{d_1}\norm{\M}, \quad \max_j\norm{\e_j\trans \V_0\U_0\trans}^2 \le \frac{ 10\mu k }{d_2}\norm{\M}\label{initial_cond_Asym}
\end{align}
\end{lemma}
We will focus mostly on Lemma \ref{lem:initial_Asym}, and prove Lemma \ref{lem:initial_Sym} as a special case. Most of the argument of this section follows from \cite{keshavan2012efficient}. We include here for completeness. The remaining of this section can be viewed as proving both the Frobenius norm claim and incoherence claim of Lemma \ref{lem:initial_Asym} seperately.

In this section, We always denote $d = \max\{d_1, d_2\}$. For simplicity, WLOG, we also assume $\norm{\M} = 1$ in all proof. 
Also, when it's clear from the context, we use $\cn$ to specifically to represent $\cn(\M)$.
Then $\sigmin(\M) = \frac{1}{\cn}$. 
Also in the proof, we always denote $\text{SVD}(\M) = \X\S\Y\trans$, and $\text{SVD}(\U\V\trans) = \W_\U \D \W_\V\trans$, where $\S$ and $\D$ are $k \times k$ diagonal matrix.

\subsection{Frobenius Norm of Initialization}

\begin{theorem}[Matrix Bernstein \cite{tropp2012user}] \label{thm:matrix_Bernstein}
A finite sequence $\{ \X_t\}$ of independent, random matrices with dimension $d_! \times d_2$.
Assume that each matrix satisfies:
\begin{equation*}
\E \X_{t} = 0, \quad \text{and} \quad  \norm{\X_t} \le R \text{~ almost surely}
\end{equation*}
Define
\begin{equation*}
\sigma^2 = \max\{\norm{\sum_t \E (\X_t\X_t\trans)}, \norm{\sum_t \E (\X_t\trans\X_t)} \}
\end{equation*}
Then, for all $s \ge 0$,
\begin{equation*}
\Pr(\norm{\sum_t \X_t} \ge s) \le (d_1+d_2) \cdot \exp (\frac{-s^2/2}{\sigma^2 + Rs/3})
\end{equation*}
\end{theorem}

\begin{lemma}\label{lem:initial_concen}
Let $|\Omega| = m$, then there exists universal constant $C, c_0$, for any
$m \ge c_0 \mu dk \log d$, with probability at least $1-\frac{1}{d^{10}}$, we have:
\begin{equation*}
\norm{\M - \frac{d_1d_2}{m} \proj_\Omega (\M)} \le C \sqrt{\frac{\mu d k \log d}{m}}
\end{equation*}
\end{lemma}
\begin{proof}
We know 
\begin{equation*}
\norm{\M - \frac{d_1d_2}{m} \proj_\Omega (\M)} = \frac{d_1d_2}{m} \norm{\proj_\Omega (\M) - \frac{m}{d_1d_2}\M }
\end{equation*}
and note:
\begin{equation*}
\proj_\Omega (\M) - \frac{m}{d_1d_2}\M = \sum_{ij} \M_{ij}(Z_{ij} - \frac{m}{d_1d_2}) \e_i \e_j\trans
\end{equation*}
where $Z_{ij}$ are independence Bernoulli$(m /d_1d_2)$ random variables. Let matrix
\begin{equation*}
\psi_{ij} =  \M_{ij}(Z_{ij} - \frac{m}{d_1d_2}) \e_i \e_j\trans 
\end{equation*}
By construction, we have:
\begin{equation*}
\norm{\sum_{ij} \psi_{ij}}
= \norm{\proj_\Omega (\M) - \frac{m}{d_1d_2}\M}
\end{equation*}

Clearly $\E \psi_{ij} = 0$. Let $\X \S \Y\trans = \text{SVD}(\M)$, then by $\mu$-incoherence of $\M$, with probability 1:
\begin{equation*}
\norm{\M}_\infty \le  \max_{ij}|\e_i\trans\X\S\Y\trans \e_j|
\le \norm{\M} \frac{\mu k}{\sqrt{d_1d_2}}
\end{equation*}
Also:
\begin{align*}
\norm{\sum_{ij}\E (\psi_{ij}\psi_{ij}\trans)}
= &\norm{ \sum_{ij}\E \M_{ij}^2(Z_{ij} - \frac{m}{d_1d_2})^2 \e_i \e_i\trans} 
\le  \frac{m}{d_1d_2}(1-\frac{m}{d_1d_2})\norm{\sum_{ij} \M_{ij}^2\e_i \e_i\trans }  \\
= & \frac{m}{d_1d_2}(1-\frac{m}{d_1d_2}) 
\max_{i} \sum_{j} \M_{ij}^2 
\le \frac{2m}{d_1d_2} \frac{\mu k}{d_1} \norm{\M}^2
= \frac{2m\mu k}{d_1^2d_2} \norm{\M}^2 \\
\norm{\sum_{ij}\E (\psi_{ij}\trans\psi_{ij})}
= &\norm{ \sum_{ij}\E \M_{ij}^2(Z_{ij} - \frac{m}{d_1d_2})^2 \e_j \e_j\trans} 
\le  \frac{m}{d_1d_2}(1-\frac{m}{d_1d_2})\norm{\sum_{ij} \M_{ij}^2\e_j \e_j\trans }  \\
= & \frac{m}{d_1d_2}(1-\frac{m}{d_1d_2}) 
\max_{j} \sum_{i} \M_{ij}^2 
\le \frac{2m}{d_1d_2} \frac{\mu k}{d_2} \norm{\M}^2
= \frac{2m\mu k}{d_1d_2^2} \norm{\M}^2
\end{align*}

Then, by matrix Bernstein (Theorem \ref{thm:matrix_Bernstein}), we have:
\begin{equation*}
\Pr(\norm{\sum_{ij} \psi_{ij}} \ge s) \le 2(d_1 + d_2) \cdot \exp (\frac{-s^2/2}{\frac{2m\mu d k}{d_1^2d_2^2} \norm{\M}^2 + \norm{\M} \frac{\mu k}{3\sqrt{d_1d_2}}s})
\end{equation*}
That is, with probability at least $1-\frac{1}{d^{10}}$, for some universal constant $C$, we have:
\begin{equation*}
\norm{\proj_\Omega (\M) - \frac{m}{d_1d_2}\M}
\le C \norm{\M} \cdot \max\{ \sqrt{\frac{m\mu d k \log d}{d^2_1d_2^2}}, \frac{\mu k\log d}{\sqrt{d_1d_2}} \}
\end{equation*}
For $m \ge \mu dk \log d$, we finishes the proof.
\end{proof}

\begin{theorem} \label{thm:initial_fnorm_Asym}
Let $\U_0\V_0\trans$ be the top-$k$ SVD of $\frac{d_1d_2}{m}\proj_{\Omega}(\M)$, where $|\Omega| =m$
then there exists universal constant $c_0$, for any
$m \ge c_0\mu dk^2  \cn^2 \log d $, with probability at least $1-\frac{1}{d^{10}}$, we have:
\begin{equation*}
\norm{\M - \U_0\V_0\trans}_F \le \frac{1}{20\cn }
\end{equation*}
\end{theorem}
\begin{proof}
Since $\M$ is a rank $k$ matrix, we know $\sigma_{k+1}(\M) = 0$, thus
\begin{equation*}
\sigma_{k+1}(\frac{d_1d_2}{m}\proj_{\Omega}(\M) ) \le \sigma_{k+1}(\M) + \norm{\frac{d_1d_2}{m}\proj_{\Omega}(\M) - \M} = \norm{\frac{d_1d_2}{m}\proj_{\Omega}(\M) - \M}
\end{equation*}
Therefore:
\begin{align*}
\norm{\M - \U_0\V_0\trans} \le& 
\norm{\M - \frac{d_1d_2}{m}\proj_{\Omega}(\M)}
+ \norm{\frac{d_1d_2}{m}\proj_{\Omega}(\M) - \U_0\V_0\trans} \nn \\
\le &\norm{\M - \frac{d_1d_2}{m}\proj_{\Omega}(\M)} + \sigma_{k+1}(\frac{d_1d_2}{m}\proj_{\Omega}(\M) )
\le 2\norm{\M - \frac{d_1d_2}{m}\proj_{\Omega}(\M)}
\end{align*}
Meanwhile, since $\rank(\M) = k$, $\rank(\U_0\V_0\trans) = k$, 
we know: $\rank(\M - \U_0\V_0\trans) \le 2k$, and therefore:
\begin{equation*}
\norm{\M - \U_0\V_0\trans}_F \le \sqrt{2k}\norm{\M - \U_0\V_0\trans} 
\le 2\sqrt{2k}\norm{\M - \frac{d_1d_2}{m}\proj_{\Omega}(\M)}
\end{equation*}
by choosing $m \ge c_0\mu dk^2 \log d \cdot \cn^2$ for large enough constant $c_0$ and apply Lemma \ref{lem:initial_concen}, we finishes the proof.
\end{proof}

\subsection{Incoherence of Initialization}
\begin{lemma} \label{lem:initial_incoherence_Asym}
Let $\U\V\trans$ be the top-$k$ SVD of $\frac{d_1d_2}{m}\proj_{\Omega}(\M)$, where $|\Omega| =m$.
then there exists universal constant $c_0$, for any
$m \ge c_0 \mu dk \kappa^2 \log d$, with probability at least $1-\frac{1}{d^{10}}$, we have:
\begin{equation*}
\max_j \norm{\e_j\trans (\M\trans - \V\U\trans)} \le 2\sqrt{\frac{\mu k}{d_2}}
\end{equation*}
\end{lemma}

\begin{proof}
Suppose $\text{SVD}(\M) = \X\S\Y\trans$. Denote $\tX = \X \S^{\frac{1}{2}}$ and 
$\tY = \Y \S^{\frac{1}{2}}$. Also let $\text{SVD}(\U\V\trans) = \W_\U \D \W_\V\trans$.

Then, we have:
\begin{align*}
&\norm{\e_j\trans (\M\trans - \V\U\trans)}
= \norm{\e_j\trans (\M\trans - \frac{d_1d_2}{m}\proj_{\Omega}(\M)\trans \W_\U\W_\U\trans)} \nn \\
= &\norm{\e_j\trans (\M\trans - \M\trans\W_\U\W_\U\trans +  \M\trans\W_\U\W_\U\trans- \frac{d_1d_2}{m}\proj_{\Omega}(\M)\trans \W_\U\W_\U\trans)} \nn \\
\le &\norm{\e_j\trans \M\trans (\I - \W_\U\W_\U\trans)}
+ \norm{\e_j\trans (\M\trans- \frac{d_1d_2}{m}\proj_{\Omega}(\M)\trans)\W_\U\W_\U\trans }
\end{align*}

For the first term, since $\W_\U\trans\W_{\U, \perp}=0$, we have:
\begin{align*}
&\norm{\e_j\trans \M\trans (\I - \W_\U\W_\U\trans)}
\le \norm{\e_j\trans\Y} \norm{\S\X\trans\W_{\U, \perp}\W_{\U, \perp}\trans} \nn \\
=& \sqrt{\frac{\mu k}{d_2}} \norm{\Y\trans(\M\trans - \W_\V \D \W_\U\trans)\W_{\U, \perp}\W_{\U, \perp}\trans}
\nn \\
\le & \sqrt{\frac{\mu k}{d_2}} \norm{\M\trans - \W_\V \D \W_\U\trans}
\le \sqrt{\frac{\mu k}{d_2}}\cdot\frac{1}{\cn}
\end{align*}
The last step is due to sample $m \ge \mu dk \kappa^2 \log d$, and theorem \ref{thm:initial_fnorm_Asym}.

For the second term, we have:
\begin{align}
&\norm{\e_j\trans (\frac{d_1d_2}{m}\proj_{\Omega}(\M)\trans- \M\trans)\W_\U\W_\U\trans }
= \norm{\tY_j\trans (\frac{d_1d_2}{m}\sum_{i:(i,j)\in \Omega} \tx_i \w_{\U, i}\trans - \sum_i\tx_i \w_{\U, i}\trans)\W_\U\trans } \nn \\
\le &\sqrt{\frac{\mu k}{d_2}} \cdot \frac{d_1d_2}{m} \cdot \norm{\sum_{i:(i,j)\in \Omega} \tx_i \w_{\U, i}\trans - \frac{m}{d_1d_2}\sum_i \tx_i \w_{\U, i}\trans} \label{ini_inco_eq1}
\end{align}
Where $\tx_i$ and $\w_{\U, i}$ are the i-th row of $\tX$ and $\W_{\U}$ respectively.

Let $\phi_{ij} = \tx_i \w_{\U, i}\trans(Z_{ij} - \frac{m}{d_1d_2})$, where
$Z_{ij}$ is Bernoulli$(\frac{m}{d_1d_2})$ random variable, $Z_{ij} = 1$ iff $(i, j) \in \Omega$.
Clearly, we have $\E \phi = 0$, and with probability 1:
\begin{align*}
\norm{\phi_{ij}} \le  2\norm{\tx_i}\norm{\w_{\U, i}} 
\le 2\sqrt{\frac{\mu k}{d_1}} \max_{i} \norm{\e_i\trans \W_{\U}}
\end{align*}
Also, we have variance term:
\begin{align*}
\norm{\sum_{i}\E \phi_{ij}\trans\phi_{ij}}
= &\norm{\sum_{i} \E (Z_{ij} - \frac{m}{d_1d_2})^2 \norm{\tx_i}^2 \w_{\U, i} \w_{\U, i}\trans} \nn \\
\le & \frac{m}{d_1d_2}(1-\frac{m}{d_1d_2})\max_i \norm{\tx_i}^2 \norm{\sum_{i}\w_{\U, i} \w_{\U, i}\trans} \nn \\
\le & \frac{m}{d_1d_2} \frac{\mu k}{d_1} \norm{\W_\U \trans \W_\U} \le \frac{\mu k m}{d_1^2 d_2} \\
\norm{\sum_i \E \phi_{ij}\phi_{ij}\trans} = &\norm{\sum_{i} \E (Z_{ij} - \frac{m}{d_1d_2})^2 \norm{\w_{\U, i}}^2 \tx_i \tx_i\trans} \nn\\
\le & \frac{m}{d_1d_2} \max_{i} \norm{\e_i\trans \W_{\U}}^2
\end{align*}

Therefore, with $m \ge \mu dk \kappa^2 \log d$, by matrix Bernstein, we have with probability at least $1-\frac{1}{d^{10}}$, we know that for all $j\in [d_2]$, there exists some absolute constant $C'$ so that:
\begin{align*}
\norm{\sum_{i:(i,j)\in \Omega} \tx_i \w_{\U, i}\trans - \frac{m}{d_1d_2}\sum_i \tx_i \w_{\U, i}\trans}
\le C'\sqrt{\frac{m \log d}{d_1d_2}}(\sqrt{\frac{\mu k}{d_1}} + \max_{i} \norm{\e_i\trans \W_{\U}})
\end{align*}
Substitue into Eq.\eqref{ini_inco_eq1}, this gives:
\begin{equation*}
\norm{\e_j\trans (\frac{d_1d_2}{m}\proj_{\Omega}(\M)\trans- \M\trans)\W_\U\W_\U\trans }
\le C'\sqrt{\frac{\mu k d_1 \log d}{m}}(\sqrt{\frac{\mu k}{d_1}} + \max_{i} \norm{\e_i\trans \W_{\U}})
\end{equation*}

On the other hand, we also have:
\begin{align*}
\norm{\e_i\trans \W_{\U}} \le &\norm{\e_i\trans \W_{\U} \S}\norm{\S^{-1}}
=2 \cn \norm{\e_i\trans\U\V\trans}
\le 2\cn (\norm{\e_i\trans(\U\V\trans-\M)} + \norm{\e_i\trans\M}) \nn \\
\le &2\cn(\sqrt{\frac{\mu k}{d_1}} + \norm{\e_i\trans(\U\V\trans-\M)})
\end{align*}
This gives overall inequality:
\begin{align*}
\max_{j}\norm{\e_j\trans (\V\U\trans- \M\trans)} \le \sqrt{\frac{\mu k}{d_2}}\cdot\frac{1}{\cn}
+ C''\sqrt{\frac{\mu k d_1 \log d}{m}}\cn(\sqrt{\frac{\mu k}{d_1}} + \max_{i}\norm{\e_i\trans(\U\V\trans-\M)})
\end{align*}
By symmetry, we will also have:
\begin{align*}
\max_{i}\norm{\e_i\trans(\U\V\trans-\M)} \le \sqrt{\frac{\mu k}{d_1}}\cdot\frac{1}{\cn}
+ C''\sqrt{\frac{\mu k d_2 \log d}{m}}\cn(\sqrt{\frac{\mu k}{d_2}} +\max_{j}\norm{\e_j\trans (\V\U\trans- \M\trans)} )
\end{align*}
Combine above two equations and choose $m \ge c_0\mu dk \kappa^2 \log d$ for some large enough $c_0$. We have:
\begin{equation*}
\max_j \norm{\e_j\trans (\M\trans - \V\U\trans)} \le 2 \sqrt{\frac{\mu k}{d_2}}
\end{equation*}
This finishes the proof.

\end{proof}

\begin{theorem} \label{thm:initial_incoherence_Asym}
Let $\U_0\V_0\trans$ be the top-$k$ SVD of $\frac{d_1d_2}{m}\proj_{\Omega}(\M)$, where $|\Omega| =m$.
then there exists universal constant $c_0$, for any
$m \ge c_0 \mu dk \kappa^2 \log d$, with probability at least $1-\frac{1}{d^{10}}$, we have:
\begin{equation*}
\max_i \norm{\e_i\trans \U_0\V_0\trans}^2 \le \frac{ 9\mu k}{d_1} \quad\text{~and~}\quad \max_j \norm{\e_j\trans \V_0\U_0\trans}^2 \le \frac{ 9\mu k}{d_2}
\end{equation*}
\end{theorem}

\begin{proof}
By Theorem \ref{lem:initial_incoherence_Asym}, we know for any $j\in[d_2]$:
\begin{equation*}
\norm{\e_j\trans (\M\trans - \V_0\U_0\trans)}\le 2 \sqrt{\frac{\mu k}{d_2}}
\end{equation*}
Therefore, we have:
\begin{align*}
\norm{\e_j\trans \V_0\U_0\trans} 
\le &[\norm{\e_j\trans\M\trans} + \norm{\e_j\trans (\M\trans - \V_0\U_0\trans)}] 
\le 3\sqrt{\frac{ \mu k }{d_2}}
\end{align*}
By symmetry, we also know for any $i\in[d_1]$
\begin{equation*}
\norm{\e_i\trans \U_0\V_0\trans} \le 3\sqrt{\frac{\mu k}{d_1}}
\end{equation*}
Which finishes the proof.
\end{proof}

For the special case where $\M \in \R^{d\times d}$ is symmetric and PSD, we can easily extends to have following:
\begin{corollary} \label{cor:initial_incoherence_sym}
Let $\U_0\U_0\trans$ be the top-$k$ SVD of $\frac{d^2}{m}\proj_{\Omega}(\M)$, where $|\Omega| =m$.
then there exists universal constant $c_0$, for any
$m \ge c_0 \mu dk \kappa^2 \log d$, with probability at least $1-\frac{1}{d^{10}}$, we have:
\begin{equation*}
\max_i \norm{\e_i\trans \U_0}^2 \le \frac{ 10\mu k \cn}{d}
\end{equation*}
\end{corollary}

\begin{proof}
By Corollary \ref{thm:initial_incoherence_Asym}, we have:
\begin{equation*}
\max_i \norm{\e_i\trans \U_0\U_0\trans}^2 \le \frac{ 9\mu k}{d}
\end{equation*}
On the other hand, by Theorem \ref{thm:initial_fnorm_Asym}, we have:
\begin{align*}
\sigmin(\U_0\trans \U_0) = \sigma_k(\U_0\U_0\trans)
\ge \sigma_k(\M) - \norm{\M - \U_0\U_0\trans} 
\ge \frac{9}{10\cn}
\end{align*}
Therefore, for any $i\in[d]$ we have:
\begin{align*}
\norm{\e_i\trans \U_0}^2 \le \frac{\norm{\e_i\trans \U_0\U_0\trans}^2}
{\sigmin(\U_0\trans \U_0)}
\le \frac{10 \mu k\cn}{d}
\end{align*}
Which finishes the proof.
\end{proof}

Finally, Lemma \ref{lem:initial_Asym} can be easily concluded from Theorem \ref{thm:initial_fnorm_Asym} and Theorem \ref{thm:initial_incoherence_Asym}, while Lemma \ref{lem:initial_Sym} is also directly proved by Theorem \ref{thm:initial_fnorm_Asym} and Corollary \ref{cor:initial_incoherence_sym}.


\section{Proof of Symmetric PSD Case}

In this section, we prove Theorem \ref{thm:main_Sym}. WLOG, we continue to assume $\norm{\M} = 1$ in all proof. 
Also, when it's clear from the context, we use $\cn$ to specifically to represent $\cn(\M)$.
Then $\sigmin(\M) = \frac{1}{\cn}$. 
Also in this section, we always denote $\text{SVD}(\M) = \X\S\X\trans$, and $\text{SVD}(\U\U\trans) = \W \D \W\trans$.

The most essential part to prove Theorem \ref{thm:main_Sym} is proving following Theorem:
\begin{theorem} [restatement of Theorem \ref{thm:localcontrol_sketch}] \label{thm:localcontrol}
Let $f(\U) = \norm{\U\U\trans - \M}_F^2$ and $g_i(\U) = \norm{\e_i\trans \U}^2$. 
Suppose after initialization, we have:
\begin{equation*}
f(\U_0) \le \left(\frac{1}{20 \cn}\right)^2, \quad\quad \max_i g_i(\U_0) \le \frac{10\mu k \cn^2}{d}
\end{equation*}
Then, there exist some absolute constant $c$ such that for any learning rate $\eta < \frac{c}{\mu dk \cn^3 \log d}$,
with at least $1-\frac{T}{d^{10}}$ probability, we will have for all $t\le T$ that:
\begin{equation*}
f(\U_t) \le (1-\frac{\eta}{2\cn})^t\left(\frac{1}{10 \cn}\right)^2, \quad\quad \max_i g_i(\U_t) \le \frac{20\mu k \cn^2}{d}
\end{equation*}
\end{theorem}
Theorem \ref{thm:localcontrol} says once initialization algorithm provides $\U_0$ in good local region, with high probability $\U_t$ will always stay in this good region and $f(\U_t)$ is linear converging to 0. With this theorem, we can then immediately conclude Theorem \ref{thm:main_Sym} from Theorem \ref{thm:localcontrol} and Lemma \ref{lem:initial_Sym}.

The rest of this section all focus on proving Theorem \ref{thm:localcontrol}.
First, we prepare with a few lemmas about the property of objective function, and the spectral property of $\U$ in a local Frobenius ball around optimal. Then, we prove Theorem \ref{thm:localcontrol} by constructing two supermartingales related to $f(\U_t), g_i(\U_t)$ each, and applying concentration argument.

For symmetric PSD case, we denote the stochastic gradient as:
\begin{equation*}
SG(\U) = 2 d^2(\U\U\trans - \M)_{ij} (\e_i\e_j\trans + \e_j\e_i\trans)\U
\end{equation*}
The update in Algorithm \ref{algo:MC_Sym} can be now written as:
\begin{equation} \label{update_sym}
\U_{t+1} \leftarrow \U_t - \eta SG(\U_t)
\end{equation}
We immediately have the property:
\begin{equation*}
\E SG(\U) = \nabla f(\U) = 4(\U\U\trans - \M)\U
\end{equation*}

\subsection{Geometric Properties in Local Region}

First, we prove two lemmas w.r.t the smoothness and property similar to strongly convex for objective function:

\begin{lemma} (restatement of Lemma \ref{lem:smoothness_sketch})\label{lem:smoothness}
Within the region $\mathcal{D} = \{\U | \norm{\U} \le \Gamma\}$, we have function $f(\U) = \fnorm{\M-\U\U\trans}^2$ satisfying for any $\U_1, \U_2 \in \mathcal{D}$:
\begin{equation*}
\fnorm{\nabla f(\U_1) - \nabla f(\U_2)} \le \beta \fnorm{\U_1 - \U_2}
\end{equation*}
where smoothness parameter $\beta = 16 \max\{\Gamma^2, \norm{\M}\}$.
\end{lemma}

\begin{proof}
Inside region $\mathcal{D}$, we have:
\begin{align}
 & \fnorm{\nabla f(\U_1) - \nabla f(\U_2)} \nn\\
 = &\fnorm{4(\U_1\U_1\trans - \M)\U_1 - 4(\U_2\U_2\trans - \M)\U_2} \nn\\
 \le& 4\fnorm{\U_1\U_1\trans\U_1 - \U_2 \U_2\trans\U_2} + 4\fnorm{\M(\U_1 - \U_2)} \nn\\
 =&4\fnorm{\U_1\U_1\trans(\U_1 - \U_2) + \U_1(\U_1 - \U_2)\trans\U_2 + (\U_1 - \U_2) \U_2\trans\U_2} 
 + 4\fnorm{\M(\U_1 - \U_2)}\nn\\
 \le & 12\max\{\norm{\U_1}^2, \norm{\U_2}^2\}\fnorm{\U_1 - \U_2} + 4\norm{\M}\fnorm{\U_1 - \U_2}\nn \\
 \le & 16 \max\{\Gamma^2, \norm{\M}\}\fnorm{\U_1 - \U_2} \nn
\end{align}
\end{proof}

\begin{lemma} (restatement of Lemma \ref{lem:pseudostronglyconvex_sketch})\label{lem:pseudostronglyconvex}
Within the region $\mathcal{D} = \{\U | \sigmin(\X\trans\U) \ge \gamma \}$, then we have function $f(\U) = \fnorm{\M-\U\U\trans}^2$ satisfying:
\begin{equation*}
\|\nabla f(\U)\|^2_F \ge \alpha f(\U)
\end{equation*}
where constant $\alpha = 4 \gamma^2$.
\end{lemma}

\begin{proof}
Inside region $\mathcal{D}$, recall we denote $\W\D\W\trans = \text{SVD}(\U\U\trans)$, thus we have:
\begin{align}
&\fnorm{\nabla f(\U)}^2  = 16\fnorm{(\U\U\trans - \M)\U}^2 \nn \\
= &16[\fnorm{\proj_\W(\U\U\trans - \M)\U}^2 + \fnorm{\proj_{\W_\perp}(\U\U\trans - \M)\U}^2] \nn \\
\ge & 16[\sigmin(\D)\fnorm{ \proj_\W(\U\U\trans - \M)\proj_\W}^2 + \fnorm{\proj_{\W_\perp} \M\U}^2] \nn \\
\ge  & 16[\sigmin(\D)\fnorm{ \U\U\trans - \proj_\W\M\proj_\W}^2 + \fnorm{\proj_{\W_\perp} \M\U}^2] \nn
\end{align}

On the other hand, we have:
\begin{align}
&\fnorm{\proj_{\W_\perp} \M \U}^2
= \fnorm{\proj_{\W_\perp} \X \mSigma \X\trans \U}^2
\ge \sigmin^2(\X\trans \U)\fnorm{\proj_{\W_\perp} \X \mSigma }^2 \nn \\
= &\sigmin^2(\X\trans \U) \tr(\proj_{\W_\perp}\M^2\proj_{\W_\perp})
=\sigmin^2(\X\trans \U)\fnorm{\proj_{\W_\perp} \M }^2 \nn
\end{align}
and 
\begin{equation*}
\sigmin(\D) = \lambda_{\min}(\U\trans\U)
\ge \lambda_{\min}(\U\trans \proj_{\X} \U) = \sigmin^2(\X\trans \U)
\end{equation*}
Therefore, combine all above, we have:
\begin{align}
&\fnorm{\nabla f(\U)}^2  
\ge  16\sigmin^2(\X\trans \U)[\fnorm{ \U\U\trans - \proj_\W\M\proj_\W}^2 + \fnorm{\proj_{\W_\perp} \M }^2] \nn \\
\ge & 4 \sigmin^2(\X\trans \U)[\fnorm{ \U\U\trans - \proj_\W\M\proj_\W}^2 + \fnorm{\proj_{\W_\perp} \M \proj_\W}^2 + \fnorm{ \proj_\W\M \proj_{\W_\perp}}^2 + \fnorm{\proj_{\W_\perp} \M \proj_{\W_\perp}}^2] \nn \\
= & 4 \sigmin^2(\X\trans \U) \fnorm{\U\U\trans - \M}^2 \nn
\end{align}

\end{proof}

Next, we show as long as we are in some Frobenious ball around optimum, then we have good spectral property over $\U$ which guarantees the preconditions for Lemma \ref{lem:smoothness} and Lemma \ref{lem:pseudostronglyconvex}.
\begin{lemma} (restatement of Lemma \ref{lem:localguarantee_sketch}) \label{lem:localguarantee}
Within the region $\mathcal{D} = \{\U | \norm{\M - \U\U\trans}_F \le \frac{1}{10}\sigma_{k}(\M)\}$, 
we have:
\begin{equation*}
\norm{\U} \le \sqrt{2\norm{\M}}, \quad\quad \sigmin(\X\trans \U) \ge \sqrt{\sigma_k(\M)/2}
\end{equation*}
\end{lemma}

\begin{proof}
For spectral norm of $\U$, we have:
\begin{equation*}
\norm{\U}^2 \le \norm{\M} + \norm{\M - \U\U\trans}
\le \norm{\M} + \norm{\M - \U\U\trans}_F \le 2 \norm{\M}
\end{equation*}
For the minimum singular value of $\U\trans \U$, we have:
\begin{align}
\sigmin(\U\trans\U) = &\sigma_k( \U\U \trans ) \ge \sigma_k (\M) - \norm{\M - \U \U\trans} \nn \\
\ge &\sigma_k( \U\U \trans ) \ge \sigma_k (\M) - \norm{\M - \U \U\trans}_F
\ge \frac{9}{10} \sigma_k(\M) \nn
\end{align}
On the other hand, we have:
\begin{align}
\frac{9}{10} \sigma_k(\M)\norm{\X_\perp \W}^2 \le & \sigmin(\D)\norm{\X_\perp \W}^2
\le \norm{\X_\perp \trans \W \mSigma \W\trans \X_\perp} \nn\\
\le &\norm{\X_\perp \trans \U\U\trans \X_\perp}_F
=\norm{\proj_{\X_\perp}(\M - \U\U \trans)\proj_{\X_\perp}}_F \nn \\
\le & \norm{\M - \U\U \trans}_F \le \frac{1}{10}\sigma_k(\M) \nn
\end{align}
Let the principal angle between $\X$ and $\W$ to be $\theta$.
This gives $\sin^2 \theta = \norm{\X_\perp\trans \W}^2 \le \frac{1}{9}$.
Thus $\cos^2 \theta = \sigmin^2(\X\trans \W) \ge \frac{8}{9}$.
Therefore:
\begin{equation*}
\sigmin^2(\X \trans \U) \ge \sigmin^2(\X\trans \W)\sigmin(\U\trans\U) \ge \sigma_k(\M)/2
\end{equation*}

\end{proof}

\subsection{Proof of Theorem \ref{thm:localcontrol}}

Now, we are ready for our key theorem. By Lemma \ref{lem:smoothness}, Lemma \ref{lem:pseudostronglyconvex}, and Lemma \ref{lem:localguarantee}, we already know the function has good property locally in the region $\mathcal{D} = \{\U | \norm{\M - \U\U\trans}_F \le \frac{1}{10}\sigma_{k}(\M)\}$ which alludes linear convergence. Then, the work remains and also the most challenging part is to prove that once we initialize inside this region, our algorithm will guarantee $\U$ never leave this region with high probability even with relatively large stepsize.
The requirement for tight sample complexity and near optimal runtime makes it more challenging, and require us to further control the incoherence of $\U_t$ over all iterates in addition to the distance $\norm{\M - \U\U\trans}_F$.



Following is our formal proof.


\begin{proof} [Proof of Theorem \ref{thm:localcontrol}]
Define event $\fE_t = \{\forall \tau\le t, f(\U_\tau) \le (1-\frac{\eta}{2\cn})^t(\frac{1}{10 \cn})^2, \max_i g_i(\U_\tau) \le \frac{20\mu k \cn^2}{d}\}$.
Theorem \ref{thm:localcontrol} is equivalent to prove event $\fE_T$ happens with high probability. The proof achieves this by contructing two supermartingales for $f(\U_t)1_{\fE_t}$ and $g_i(\U_t)1_{\fE_t}$ (where $1_{(\cdot)}$ denote indicator function), applies concentration argument.

The proofs follow the structure of:
\begin{enumerate}
 \item The constructions of supermartingales
 \item Their probability 1 bound and variance bound in order to apply Azuma-Bernstein inequality
 \item Final combination of concentration results to conclude the proof
\end{enumerate}


First, let filtration $\fF_t = \sigma\{SG(\U_0), \cdots, SG(\U_{t-1})\}$ where $\sigma\{\cdot\}$ denotes the sigma field. Note by definiton of $\fE_t$, we have $\fE_t \subset \fF_t$. Also $\fE_{t+1} \subset \fE_t$, and thus $1_{\fE_{t+1}} \le 1_{\fE_{t}}$. Note $\fE_t$ denotes the event which up to time $t$, $\U_\tau$ always stay in a local region which both close to $\M$ and incoherent.

By Lemma \ref{lem:localguarantee}, we immediately know that conditioned on $\fE_t$, we have $\norm{\U_t} \le \sqrt{2}$, $\sigma_{\min}(\X\trans\U_t) \ge 1/\sqrt{2\cn}$ and $\sigma_{\min}(\U_t\trans\U_t) \ge 1/2\cn$. We will use this fact throughout the proof.

\paragraph{Construction of supermartingale $G$:}
Since $g_i(\U) = \e_i\trans \U\U\trans \e_i$ is a quadratic function, we know for any change $\Delta \U$, we have:
\begin{align}
g_i(\U+\Delta\U) = g_i(\U) + 2 \e_i\trans (\Delta \U ) \U\trans \e_i
+ \norm{\e_i\trans \Delta\U}^2 \nn
\end{align}


We know for any $l \in [d]$:
\begin{align}
&\E \|\e_l\trans SG(\U)\|^21_{\fE_t} 
\le \E 16d^4\delta_{il}(\u_i\trans \u_j - \M_{ij})^2 \max_i\norm{\e_i \trans \U}^21_{\fE_t} \nn \\ 
=&16d^2\norm{\e_l\trans(\U\U\trans  - \M)}^2 \max_i\norm{\e_i \trans \U}^21_{\fE_t}
\le O(\mu^2k^2 \cn^4)1_{\fE_t}  \nn
\end{align}

Therefore, by update Eq.\eqref{update_sym}, and $\E SG(\U) = \nabla f(\U) = 4(\U\U\trans - \M)\U$, we know:
\begin{align}
&\E [g_i(\U_{t+1})1_{\fE_t}|\fF_t] \nn\\
=& [g_i(\U_t) - 2\eta\e_i \trans [\E SG(\U_t)]\U_t\trans\e_i + \frac{\eta^2}{2} \E \norm{\e_i\trans SG(\U_t)}^2]1_{\fE_t} \nn \\
= & [\tr(\U_t\trans \e_i \e_i\trans [\I- 8\eta(\U_t\U_t\trans - \M)]\U_t)+ \frac{\eta^2}{2} \E \norm{\e_i\trans SG(\U_t)}^2]1_{\fE_t} \nn \\
= &[\tr(\U_t\trans \e_i \e_i\trans \U_t (\I- 8\eta\U_t\trans\U_t)) + 8\eta\tr(\U_t\trans \e_i \e_i \trans \M \U_t)+ \eta^2 O(\mu^2 k^2\cn^4)]1_{\fE_t} \nn \\
\le &[(1-8\eta \sigmin(\U_t\trans \U_t))g_i(\U_t)+ 8\eta\tr(\U_t\trans \e_i \e_i \trans \M \U_t)+ \eta^2 O(\mu^2 k^2\cn^4)]1_{\fE_t} \nn\\
\le &[(1-\frac{4\eta}{\cn})g_i(\U_t)+ 16\sqrt{10}\frac{\eta \mu k \cn}{d}+ \eta^2 O(\mu^2 k^2\cn^4)]1_{\fE_t} \nn\\
\le & [(1-\frac{4\eta}{\cn})g_i(\U_t) + 60\frac{\eta \mu k\cn}{d})]1_{\fE_t}
\nn
\end{align}
The last step is true by choosing constant $c$ in learning rate $\eta$ to be small enough.

Let $G_{it} = (1-\frac{4\eta}{\cn})^{-t}(g_i(\U_{t})1_{\fE_{t-1}} - 15\frac{\mu k\cn^2}{d})$. This gives:
\begin{equation*}
\E G_{i(t+1)}
\le (1-\frac{4\eta}{\cn})^{-t}(g_i(\U_{t})1_{\fE_{t}} - 15\frac{\mu k\cn^2}{d})
\le G_{it}
\end{equation*}
That is $G_{it}$ is supermartingale.

\paragraph{Probability 1 bound for $G$:}
We also know
\begin{align}
G_{it} - \E [G_{it}|\fF_{t-1}] = & (1-\frac{4\eta}{\cn})^{-t}
\left[-\eta\e_i \trans [SG(\U_t)-\E SG(\U_t)]\U_t\trans\e_i\right.  \nn \\
& + \left.\frac{\eta^2}{2} [\|\e_i\trans SG(\U_t)\|^2- \E \|\e_i\trans SG(\U_t)\|^2]\right]1_{\fE_{t-1}} \label{decomp_G}
\end{align}

Since when sample $(i,j)$ entry of matrix $\M$, for any $l\in[d]$, we have:
\begin{align}
&\e_l \trans [SG(\U_t)]\U_t\trans\e_l \cdot 1_{\fE_{t-1}} = O(1)\tr(\U\trans\e_l \e_l \trans SG(\U_t))1_{\fE_{t-1}}\nn\\
=& O(1)d^2(\U\U\trans  - \M)_{ij}\tr[\U\trans\e_l \e_l\trans (\vect{e}_i\u_j\trans + \vect{e}_j\u_i\trans)]1_{\fE_{t-1}}
\nn\\
\le &O(1)d^2 \norm{\U\U\trans - \M}_\infty \max_i\norm{\e_i \trans \U}^21_{\fE_{t-1}}
\le O(\mu^2 k^2 \cn^4)1_{\fE_{t-1}} \nn
\end{align}
and 
\begin{align}
&\|\e_l\trans SG(\U_t)\|^2 1_{\fE_{t-1}}
= O(1)\norm{\e_l\trans SG(\U_t)}^21_{\fE_{t-1}} \nn\\
=& O(1)d^4(\U\U\trans  - \M)^2_{ij}\norm{\e_l\trans (\vect{e}_i\u_j\trans + \vect{e}_j\u_i\trans)}^21_{\fE_{t-1}} \nn\\
\le & O(1) d^4\norm{\U\U\trans - \M}^2_\infty \max_i\norm{\e_i \trans \U}^21_{\fE_{t-1}}
\le O(\mu^3d k^3 \cn^6)1_{\fE_{t-1}}\nn
\end{align}
Therefore, by Eq.\eqref{decomp_G}, we have with probability 1:
\begin{equation} \label{p1_G}
|G_{it} - \E [G_{it}|\fF_{t-1}]| \le (1-\frac{4\eta}{\cn})^{-t} \eta O(\mu^2k^2 \cn^4)1_{\fE_{t-1}}
\end{equation}

\paragraph{Variance bound for $G$:}
For any $l\in [d]$, we also know
\begin{align}
&\Var(\e_l \trans [SG(\U_t)]\U_t\trans\e_l \cdot 1_{\fE_{t-1}} |\fF_{t-1} )
\le \E [\langle \nabla g_l(\U_t), SG(\U_t)\rangle^2 1_{\fE_{t-1}} |\fF_{t-1}] \nn \\
= &O(1)\frac{1}{d^2} \sum_{ij}d^4(\U\U\trans  - \M)^2_{ij}\tr[\U\trans\e_l \e_l\trans (\vect{e}_i\u_j\trans + \vect{e}_j\u_i\trans)]^21_{\fE_{t-1}} \nn \\
\le &O(1) d^2 \sum_{j} (\U\U\trans  - \M)^2_{lj}\tr[\U\trans\e_l \u_j\trans]^2 1_{\fE_{t-1}}\nn \\
\le &O(1) d^2 \norm{\e_l\trans (\U\U\trans - \M)}^2 \max_i\norm{\e_i \trans \U}^41_{\fE_{t-1}}
\le O(\frac{\mu^3 k^3 \cn^6}{d})1_{\fE_{t-1}} \nn
\end{align}
and
\begin{align}
&\Var(\|\e_l\trans SG(\U)\|^2 1_{\fE_{t-1}} |\fF_{t-1} )
\le \E [\nabla^2 g_k(SG(\U_t), SG(\U_t))^2 1_{\fE_{t-1}} |\fF_{t-1}] \nn \\
= &O(1)\frac{1}{d^2} \sum_{ij}d^8(\U\U\trans  - \M)^4_{ij}\norm{\e_k\trans (\vect{e}_i\u_j\trans + \vect{e}_j\u_i\trans)}^4 1_{\fE_{t-1}} \nn \\
\le &O(1) d^6 \sum_{j} (\U\U\trans  - \M)^4_{kj}\norm{\u_j}^4 1_{\fE_{t-1}}\nn \\
\le &O(1) d^6 \norm{\U\U\trans - \M}^2_\infty \norm{\e_k\trans (\U\U\trans - \M)}^2 \max_i\norm{\e_i \trans \U}^41_{\fE_{t-1}}
\le O(\mu^5 d k^5 \cn^{10})1_{\fE_{t-1}} \nn
\end{align}
Therefore, by Eq.\eqref{decomp_G}, we have
\begin{equation} \label{var_G}
\Var(G_{it} | \fF_{t-1}) \le (1-\frac{4\eta}{\cn})^{-2t} \eta^2 O(\frac{\mu^3 k^3 \cn^6}{d})1_{\fE_{t-1}}
\end{equation}

\paragraph{Bernstein's inequality for $G$:} 
Let $\sigma^2 = \sum_{\tau=1}^t \Var(G_{i\tau} | \fF_{\tau-1})$, and 
$R$ satisfies, with probability 1 that $|G_{i\tau} - \E [G_{i\tau}|\fF_{\tau-1}]| \le R, ~\tau = 1, \cdots, t$. Then 
By standard Bernstein concentration inequality, we know:
\begin{equation*}
P(G_{it} \ge G_{i0} + s ) \le \exp(\frac{s^2/2}{\sigma^2 + Rs / 3})
\end{equation*}

Since $G_{i0} = g_i(\U_0) - 15 \frac{\mu k \cn^2}{d}$, let $s' = O(1)(1- \frac{4\eta}{\cn})^t [\sqrt{\sigma^2 \log d} + R\log d]$, 
we know 
\begin{equation*}
P\left(g_i(\U_t)1_{\fE_{t-1}} \ge 15 \frac{\mu k \cn^2}{d}
+ (1- \frac{4\eta}{\cn})^t(g_i(\U_0) - 15 \frac{\mu k \cn^2}{d}) + s'\right) \le \frac{1}{2d^{11}}
\end{equation*}

By Eq.\eqref{p1_G}, we know $R = (1-\frac{4\eta}{\cn})^{-t} \eta O(\mu^2k^2 \cn^4)$ satisfies that $|G_{i\tau} - \E [G_{i\tau}|\fF_{\tau-1}]| \le R, ~\tau = 1, \cdots, t$. Also by Eq. \eqref{var_G}, we have:
\begin{align*}
(1- \frac{4\eta}{\cn})^t \sqrt{\sigma^2 \log d}
\le \eta O(\sqrt{\frac{\mu^3 k^3 \cn^6 \log d}{d}})\sqrt{\sum_{\tau=1}^t(1- \frac{4\eta}{\cn})^{2t-2\tau} }
\le \sqrt{\eta} O(\sqrt{\frac{\mu^3 k^3 \cn^7 \log d}{d}})
\end{align*}
by $\eta < \frac{c}{\mu dk \cn^3 \log d}$ and choosing $c$ to be small enough, we have:
\begin{equation*}
s' = \sqrt{\eta} O(\sqrt{\frac{\mu^3 k^3 \cn^7 \log d}{d}}) + \eta O(\mu^2k^2 \cn^4 \log d)
\le \frac{\mu k \cn^2}{d} 
\end{equation*}
Since initialization gives $\max_i g_i(\U_0) \le \frac{10\mu k \cn^2}{d}$, therefore:
\begin{equation*}
P(g_i(\U_t)1_{\fE_{t-1}} \ge 20\frac{\mu k \cn^2}{d}) \le \frac{1}{2d^{11}}
\end{equation*}
That is equivalent to:
\begin{equation}
P(\fE_{t-1} \cap \{g_i(\U_t)\ge 20\frac{\mu k \cn^2}{d}\}) \le \frac{1}{2d^{11}} \label{concen_G}
\end{equation}

\paragraph{Construction of supermartingale $F$:}
On the other hand, we also have
\begin{align}
& \E \|SG(\U_t)\|^2_F1_{\fE_t} 
\le \E 16d^4(\u_i\trans \u_j - \M_{ij})^2 \max_i\norm{\e_i \trans \U}^2 1_{\fE_t} \nn \\
\le &16d^2\|\U_t \U_t\trans - \M\|_F^2 \max_i\norm{\e_i \trans \U_t}^2 1_{\fE_t}
\le O(\mu dk \cn^2)f(\U_t)1_{\fE_t} \nn
\end{align}
Therefore, by update function Eq.\eqref{update_sym},
\begin{align}
&\E [f(\U_{t+1})1_{\fE_t} |\fF_t] \nn\\
\le& [f(\U_t) - \E \langle \nabla f(\U_t), \eta SG(\U_t)\rangle + \eta^2 \E \norm{SG(\U_t)}_F^2]1_{\fE_t} \nn \\
= & [f(\U_t) - \eta \norm{\nabla f(\U_t)}_F^2 + \eta^2 \E \norm{SG(\U_t)}_F^2]1_{\fE_t} \nn \\
\le& [(1-\frac{2\eta}{\cn}) f(\U_t) + \eta^2 O(\mu dk \cn^2)f(\U_t)]1_{\fE_t} \nn\\
\le& (1-\frac{\eta}{\cn}) f(\U_t)1_{\fE_t} \nn
\end{align}

Let $F_t =  (1-\frac{\eta}{\cn})^{-t}f(\U_t)1_{\fE_{t-1}}$, we know 
$F_t$ is also a supermartingale.

\paragraph{Probability 1 bound for $F$:}

With probabilty 1, we also have:
\begin{align}
F_{t} - \E [F_{t}|\fF_{t-1}] = & (1-\frac{\eta}{\cn})^{-t}
[-\eta\langle \nabla f(\U_t), SG(\U_t) - \E SG(\U_t) \rangle \nn \\
& + \frac{\eta^2}{2} (\nabla^2 f(\zeta_t)(SG(\U_t), SG(\U_t)) - \E \nabla^2 f(\zeta_t)(SG(\U_t), SG(\U_t)))]1_{\fE_{t-1}} \label{decomp_F}
\end{align}
where $\zeta_t$ depends on $SG(\U_t)$.

First, recall we denote $\text{SVD}(\M) = \X\S\X\trans$, and $\text{SVD}(\U\U\trans) = \W \D \W\trans$ , and observe that:
\begin{align}
&\norm{\U\U\trans - \M}_{\infty}1_{\fE_{t-1}} = \max_{ij} |\tr(\e_i\trans (\U\U\trans - \M) \e_j)|1_{\fE_{t-1}} \nn\\
= &\max_{ij} |\tr(\e_i\trans (\proj_{\X} + \proj_{\X_\perp}) (\U\U\trans - \M) \e_j)|1_{\fE_{t-1}} \nn\\
\le & \max_{ij} |\tr(\e_i\trans \proj_{\X}  (\U\U\trans - \M) \e_j)|1_{\fE_{t-1}}
+ \max_{ij} |\tr(\e_i\trans \proj_{\X_\perp} \U\U\trans \e_j)|1_{\fE_{t-1}} \nn \\
\le & \max_{i}\norm{\e_i\trans \X} \norm{\U\U\trans-\M}_F1_{\fE_{t-1}} + \max_{i}\norm{\e_i\trans \W} \norm{\U\U\trans-\M}_F1_{\fE_{t-1}} \nn \\
\le & O(\sqrt{\frac{\mu k \cn^3}{d}})\sqrt{f(\U_t)} \nn
\end{align}

Then, when sample $(i,j)$ entry of matrix $\M$, we have:
\begin{align}
&\langle \nabla f(\U_t), SG(\U_t)\rangle 1_{\fE_{t-1}}
\le O(1)\norm{\nabla f(\U_t)}_F\norm{SG(\U_t)}_F1_{\fE_{t-1}} \nn \\
\le & O(1)d^2\sqrt{f(\U_t)}(\U\U\trans  - \M)_{ij}\norm{\vect{e}_i\u_j\trans + \vect{e}_j\u_i\trans}_F^21_{\fE_{t-1}} \nn \\
\le & O(1)d^2 \sqrt{f(\U_t)}\norm{\U\U\trans-\M}_\infty  \max_i \norm{\e_i\trans \U}1_{\fE_{t-1}}
\le O(\mu d k \cn^{2.5})f(\U_t)1_{\fE_{t-1}} \nn
\end{align}
and
\begin{align}
&\nabla^2 f(\zeta_t)(SG(\U_t), SG(\U_t))1_{\fE_{t-1}}
\le O(1) \norm{SG(\U_t)}_F^2 \nn \\
\le& O(1) d^4\norm{\U\U\trans-\M}^2_\infty  \max_i \norm{\e_i\trans \U}^2 1_{\fE_{t-1}}
\le O(\mu^2 d^2 k^2 \cn^5) f(\U_t) 1_{\fE_{t-1}} \nn 
\end{align}
Therefore, by decomposition Eq.\eqref{decomp_F}, we have with probability 1:
\begin{equation}
|F_{t} - \E [F_{t}|\fF_{t-1}]| \le (1-\frac{\eta}{\cn})^{-t} \eta O(\mu d k \cn^{2.5}) f(\U_{t-1})1_{\fE_{t-1}}
\le (1-\frac{\eta}{\cn})^{-t}(1-\frac{\eta}{2\cn})^t \eta O(\mu d k \cn^{0.5})1_{\fE_{t-1}} \label{p1_F}
\end{equation}

\paragraph{Variance bound for $F$:} We also know
\begin{align}
&\Var(\langle \nabla f(\U_t), SG(\U_t)\rangle 1_{\fE_{t-1}}|\fF_{t-1})
\le \E [\langle \nabla f(\U_t), SG(\U_t)\rangle^2 1_{\fE_{t-1}}|\fF_{t-1}] \nn \\
\le & \norm{\nabla f(\U_t)}_F^2 \E \norm{SG(\U_t)}_F^21_{\fE_{t-1}}
\le O(1)d^2\max_i\norm{\e_i \trans \U}^2 f^2(\U_{t-1})1_{\fE_{t-1}} \nn \\
\le & O(\mu d k \cn^2) f^2(\U_{t-1})1_{\fE_{t-1}} \nn
\end{align}
and 
\begin{align}
&\Var(\nabla^2 f(\zeta_t)(SG(\U_t), SG(\U_t)) 1_{\fE_{t-1}}|\fF_{t-1})
\le \E [\nabla^2 f(\zeta_t)(SG(\U_t), SG(\U_t))^2 1_{\fE_{t-1}}|\fF_{t-1}] \nn \\
\le & O(1) \E\norm{SG(\U_t)}_F^4
= O(1)\E d^8(\U\U\trans - \M)_{ij}^4 \max_i\norm{\e_i \trans \U}^4 1_{\fE_{t-1}} \nn \\
\le & O(1) d^6 \norm{\U\U\trans - \M}_\infty^2 \norm{\U\U\trans - \M}_F^2
\max_i\norm{\e_i \trans \U}^4 1_{\fE_{t-1}} \nn \\
\le & O(\mu^3 d^3 k^3 \cn^7)f^2(\U_{t-1})1_{\fE_{t-1}}\nn
\end{align}
Therefore, by decomposition Eq.\eqref{decomp_F}, we have:
\begin{equation}
\Var(F_t|\fF_{t-1}) \le (1-\frac{\eta}{\cn})^{-2t} \eta^2 O(\mu d k \cn^2) f^2(\U_{t-1})1_{\fE_{t-1}}
\le (1-\frac{\eta}{\cn})^{-2t} (1-\frac{\eta}{2\cn})^{2t}\eta^2 O(\frac{\mu d k}{\cn^2}) 1_{\fE_{t-1}} \label{var_F}
\end{equation}

\paragraph{Bernstein's inequality for $F$:} 
Let $\sigma^2 = \sum_{\tau=1}^t \Var(F_{\tau} | \fF_{\tau-1})$, and 
$R$ satisfies, with probability 1 that $|F_{\tau} - \E [F_{\tau}||\fF_{\tau-1}]| \le R, ~\tau = 1, \cdots, t$. Then 
By standard Bernstein concentration inequality, we know:
\begin{equation*}
P(F_{t} \ge F_{0} + s ) \le \exp(\frac{s^2/2}{\sigma^2 + Rs / 3})
\end{equation*}
Let $s' = O(1)(1- \frac{\eta}{\cn})^t [\sqrt{\sigma^2 \log d} + R\log d]$, this gives:
\begin{equation*}
P(f(\U_t)1_{\fE_{t-1}} \ge (1- \frac{\eta}{\cn})^t f(\U_0)+ s') \le \frac{1}{2d^{10}}
\end{equation*}
By Eq.\eqref{p1_F}, we know $R = (1-\frac{\eta}{\cn})^{-t} (1-\frac{\eta}{2\cn})^t \eta O(\mu d k \cn^{0.5})$ satisfies that $|F_{\tau} - \E [F_{\tau}|\fF_{\tau-1}]| \le R, ~\tau = 1, \cdots, t$. Also by Eq. \eqref{var_F}, we have:
\begin{align}
&(1- \frac{\eta}{\cn})^t \sqrt{\sigma^2 \log d}
\le \eta O(\sqrt{\frac{\mu d k \log d}{\cn^2}})\sqrt{\sum_{\tau=1}^t(1- \frac{\eta}{\cn})^{2t-2\tau}(1- \frac{\eta}{2\cn})^{2\tau}  } \nn \\
\le & (1-\frac{\eta}{2\cn})^t
\eta O(\sqrt{\frac{\mu d k \log d}{\cn^2}})\sqrt{\sum_{\tau=1}^t(1- \frac{\eta}{\cn})^{2t-2\tau}(1- \frac{\eta}{2\cn})^{2\tau-2t}  } 
\le (1-\frac{\eta}{2\cn})^t \sqrt{\eta} O(\sqrt{\frac{\mu d k \log d}{\cn}}) \nn
\end{align}
by $\eta < \frac{c}{\mu dk \cn^3 \log d}$ and choosing $c$ to be small enough, we have:
\begin{equation*}
s' = (1-\frac{\eta}{2\cn})^t [\sqrt{\eta} O(\sqrt{\frac{\mu d k \log d}{\cn}}) + \eta O(\mu d k \cn^{0.5})]
\le (1-\frac{\eta}{2\cn})^t(\frac{1}{20\cn})^2
\end{equation*}
Since $F_{0} = f(\U_0) \le (\frac{1}{20\cn})^2$, therefore:
\begin{equation*}
P(f(\U_t)1_{\fE_{t-1}} \ge (1-\frac{\eta}{2\cn})^t(\frac{1}{10\cn})^2) \le \frac{1}{2d^{10}}
\end{equation*}
That is equivalent to:
\begin{equation}
P(\fE_{t-1} \cap \{f(\U_t)\ge (1-\frac{\eta}{2\cn})^t(\frac{1}{10\cn})^2\}) \le \frac{1}{2d^{10}} \label{concen_F}
\end{equation}

\paragraph{Probability for event $\fE_T$:} 

Finally, combining the concentration result for martingale $G$ (Eq.\eqref{concen_G}) and martingale $F$ (Eq.\eqref{concen_F}), we conclude:
\begin{align}
&P(\fE_{t-1} \cap \bar{\fE_t})
=  P\left[\fE_{t-1} \cap \left(\cup_i \{g_i(\U_t) \ge 20 \frac{\mu k \cn^2}{d} \}\cup \{f(\U_t)\ge (1-\frac{\eta}{2\cn})^t(\frac{1}{10\cn})^2\}\right)\right]
\nn \\
\le & \sum_{i=1}^d P(\fE_{t-1} \cap \{g_i(\U_t)\ge 20\frac{\mu k \cn^2}{d}\})
+ P(\fE_{t-1} \cap \{f(\U_t)\ge (1-\frac{\eta}{2\cn})^t(\frac{1}{10\cn})^2\}) \le \frac{1}{d^{10}} \nn
\end{align}
Since
\begin{equation*}
P(\bar{\fE_T}) = \sum_{t=1}^T P(\fE_{t-1} \cap \bar{\fE_t}) \le \frac{T}{d^{10}}
\end{equation*}
We finishes the proof.
\end{proof}



\section{Proof of General Asymmetric Case}\label{sec:proof_Asym}
In this section, we first prove Lemma \ref{lem:equiv}, set up the equivalence between Algorithm \ref{algo:MC_Asym} and Algorithm \ref{algo:MC_Asym_Prac}. Then we 
prove the main theorem for general asymmetric matrix (Theorem \ref{thm:main_Asym}). 
WLOG, we continue to assume $\norm{\M} = 1$ in all proof. 
Also, when it's clear from the context, we use $\cn$ to specifically to represent $\cn(\M)$.
Then $\sigmin(\M) = \frac{1}{\cn}$. 
Also in this section, we always use $d=\max\{d_1, d_2\}$ and denote $\text{SVD}(\M) = \X\S\Y\trans$, and $\text{SVD}(\U\V\trans) = \W_U \D \W_V\trans$.

\begin{proof}[Proof of Lemma \ref{lem:equiv}]
Let us always denote the iterates in Algorithm \ref{algo:MC_Asym} by $\U_t$, $\V_t$, and denote the corresponding iterates in Algorithm \ref{algo:MC_Asym_Prac} by $\U_t'$, $\V_t'$ using prime version.
We use induction to prove the equivalence. Assume at time $t$ we have $\U_t\V_t\trans = \U_t'\V_t'{}\trans$.
Recall in Algorithm \ref{algo:MC_Asym}, we renormalize $\U_t, \V_t$ to $\tU$, $\tV$, this set up the correspondence:
\begin{align*}
\tU &= \U'_t\mR'_U\D'_U{}^{-\frac{1}{2}}\mQ'_U\D'{}^{\frac{1}{2}} \\
\tV &= \V'_t\mR'_V\D'_V{}^{-\frac{1}{2}}\mQ'_V\D'{}^{\frac{1}{2}}
\end{align*}
Denote $\mP'_U  = \mR'_U\D'_U{}^{-\frac{1}{2}}\mQ'_U\D'{}^{\frac{1}{2}}$, 
and $\mP'_V = \V'_t\mR'_V\D'_V{}^{-\frac{1}{2}}\mQ'_V\D'{}^{\frac{1}{2}}$.
Clearly $\mP'_U \mP'_V{}\trans = \I$. 
Then we have $\tU =\U'_t \mP'_U$, $\tV = \mP'_V$ and thus:
\begin{align}
&\U_{t+1}\V_{t+1}\trans  \nn\\
=&(\tU - 2\eta d_1d_2(\tU\tV\trans - \M)_{ij} \vect{e}_i\vect{e}_j\trans\tV)(\tV - 2\eta d_1d_2(\tU\tV\trans - \M)_{ij} \vect{e}_j\vect{e}_i\trans\tU)\trans \nn\\
=& ( \U'_t\mP'_U - 2\eta d_1d_2(\U'_{t}\V'_{t}{}\trans - \M)_{ij} \vect{e}_i\vect{e}_j\trans\V'_{t}\mP'_V)
(\V'_t\mP'_V - 2\eta d_1d_2(\U'_{t}\V'_{t}{}\trans - \M)_{ij} \vect{e}_j\vect{e}_i\trans\U'_{t}\mP'_U)\trans \nn \\
=& ( \U'_t - 2\eta d_1d_2(\U'_{t}\V'_{t}{}\trans - \M)_{ij} \vect{e}_i\vect{e}_j\trans\V'_{t}\mP'_V\mP'_U{}^{-1}) \nn\\
&\cdot (\V'_t - 2\eta d_1d_2(\U'_{t}\V'_{t}{}\trans - \M)_{ij} \vect{e}_j\vect{e}_i\trans\U'_{t}\mP'_U \mP'_V{}^{-1})\trans \nn \\
=&\U'_{t+1}\V'_{t+1}{}\trans \nn
\end{align}
Clearly with same initialization algorithm, we have $\U_0\V_0\trans = \U'_0\V'_0{}\trans$, by induction, we finish the proof.
\end{proof}

Now we proceed to prove Theorem \ref{thm:main_Asym}.
Since Algorithm \ref{algo:MC_Asym} and Algorithm \ref{algo:MC_Asym_Prac} are equivalent, we will focus our analysis on Algorithm \ref{algo:MC_Asym} which is more theoretical appealing.
As for the symmetric PSD case, we first present the essential ingradient:

\begin{theorem}\label{thm:localcontrol_Asym}
Let $f(\U, \V) = \norm{\U\V\trans - \M}_F^2$, $g_i(\U, \V) = \norm{\e_i\trans \U\V\trans}^2$,
and $h_j(\U,  \V) = \norm{\e_j\trans\V\U\trans}^2$, for $i \in [d_1]$ and $j \in [d_2]$.
Suppose after initialization, we have:
\begin{equation*}
f(\U_0, \V_0) \le (\frac{1}{20 \cn})^2, \quad\quad \max_i g_i(\U_0, \V_0) \le \frac{10\mu k \cn^2}{d_1}, \quad\quad \max_j h_j(\U_0, \V_0) \le \frac{10\mu k \cn^2}{d_2} 
\end{equation*}
Then, there exist some absolute constant $c$ such that for any learning rate $\eta < \frac{c}{\mu dk \cn^3 \log d}$,
with at least $1-\frac{T}{d^{10}}$ probability, we will have for all $t\le T$ that:
\begin{equation*}
f(\U_t, \V_t) \le (1-\frac{\eta}{2\cn})^t(\frac{1}{10 \cn})^2, \quad \max_i g_i(\U_t, \V_t) \le \frac{100\mu k \cn^2}{d_1}, \quad \max_j h_j(\U_t, \V_t) \le \frac{100\mu k \cn^2}{d_2}
\end{equation*}
\end{theorem}

Theorem \ref{thm:main_Asym} can easily be concluded from Theorem \ref{thm:localcontrol_Asym} and Lemma \ref{lem:initial_Asym}.
Theorem \ref{thm:localcontrol_Asym} also provides similar guarantees as Theorem \ref{thm:localcontrol} in symmetric case. However, due to the additional invariance between $\U$ and $\V$, Theorem \ref{thm:localcontrol_Asym} need to keep track of more complicated potential function $g_i(\U, \V)$ and $h_j(\U, \V)$ to control the incoherence, which makes the proof more involved.

The rest of this section all focus on proving Theorem \ref{thm:localcontrol_Asym}.
Similar to symmetric PSD case, we also first prepare with a few lemmas about the property of objective function, and the spectral property of $\U, \V$ in a local Frobenius ball around optimal. Then, we prove Theorem \ref{thm:localcontrol_Asym} by constructing three supermartingales related to $f(\U_t, \V_t), g_i(\U_t, \V_t), h_j(\U_t, \V_t)$ each, and applying concentration argument.

To make the notation clear, denote gradient $\grad f(\U, \V) \in \R^{(d_1+d_2) \times k}$:
\begin{equation*}
\grad f(\U, \V) = 
\begin{pmatrix}
\pgrad{\U} f(\U, \V) \\
\pgrad{\V} f(\U, \V)
\end{pmatrix}
\end{equation*}
Also denote the stochastic gradient $SG(\U, \V)$ by (if we sampled entry $(i, j)$ of matrix $\M$)
\begin{align*} \label{eq_SG_asym}
SG(\U, \V) &= 
2d_1d_2(\U\V\trans - \M)_{ij}\begin{pmatrix}
\e_i\e_j\trans \V \\
\e_j\e_i\trans \U
\end{pmatrix} \\
\E SG(\U, \V) & = \grad f(\U, \V) = 2
\begin{pmatrix}
(\U\V\trans - \M)\V \\
(\U\V\trans - \M)\trans \U
\end{pmatrix} 
\end{align*}
By update function, we know:
\begin{equation*}
\begin{pmatrix}
\U_{t+1} \\ \V_{t+1}
\end{pmatrix} 
= 
\begin{pmatrix}
\tU \\ \tV
\end{pmatrix} 
-\eta SG(\tU, \tV)
\end{equation*}
and $\tU\tV\trans = \U_t \V_t\trans$ is the renormalized version of $\U_t \V_t\trans$.

\subsection{Geometric Properties in Local Region}
Similar to symmetric PSD case, we first prove two lemmas w.r.t the smoothness and property similar to strongly convex for objective function:

\begin{lemma} \label{lem:smoothness_g}
Within the region $\mathcal{D} = \{(\U, \V) | \norm{\U} \le \Gamma, \norm{\V} \le \Gamma\}$, we have function $f(\U, \V) = \fnorm{\M-\U\V\trans}^2$ satisfying:
\begin{equation*}
\fnorm{\nabla f(\U_1, \V_1) - \nabla f(\U_2, \V_2)}^2 \le \beta^2 (\fnorm{\U_1 - \U_2}^2 + \fnorm{\V_1 - \V_2}^2)
\end{equation*}
where smoothness parameter $\beta = 8 \max\{\Gamma^2, \norm{\M}\}$.
\end{lemma}

\begin{proof}
Inside region $\mathcal{D}$, we have:
\begin{align}
& \fnorm{\nabla f(\U_1, \V_1) - \nabla f(\U_2, \V_2)}^2 \nn\\
= & \fnorm{\pgrad{\U} f(\U_1, \V_1) - \pgrad{\U} f(\U_2, \V_2)}^2 
+ \fnorm{\pgrad{\V} f(\U_1, \V_1) - \pgrad{\V} f(\U_2, \V_2)}^2 \nn \\
= & 4(\fnorm{(\U_1\V_1\trans - \M) \V_1 - (\U_2\V_2\trans - \M) \V_2}^2 + \fnorm{(\U_1\V_1\trans - \M)\trans \U_1 - (\U_2\V_2\trans - \M)\trans \U_2}^2)\nn \\
\le & 64\max\{\Gamma^4, \norm{\M}^2\}(\fnorm{\U_1 - \U_2}^2 + \fnorm{\V_1-\V_2}^2) \nn
\end{align}
The last step is by similar technics as in the proof of Lemma \ref{lem:smoothness}, by expanding 
\begin{equation*}
\U_1 \V_1\trans \V_1 - \U_2 \V_2 \trans \V_2 = 
(\U_1 - \U_2) \V_1\trans \V_1 + \U_2(\V_1 - \V_2)\trans \V_1+ \U_2\V_2\trans(\V_1 - \V_2)
\end{equation*}

\end{proof}

\begin{lemma} \label{lem:pseudostronglyconvex_g}
Within the region $\mathcal{D} = \{(\U, \V) | \sigmin(\X\trans\U) \ge \gamma, \sigmin(\Y\trans\V) \ge \gamma \}$, then we have function $f(\U, \V) = \fnorm{\M-\U\V\trans}^2$ satisfying:
\begin{equation*}
\|\nabla f(\U, \V)\|^2_F \ge \alpha f(\U, \V)
\end{equation*}
where constant $\alpha = 4 \gamma^2$.
\end{lemma}

\begin{proof}
Let $\hat{\U}, \hat{\V}$ be the left singular vectors of $\U, \V$.
Inside region $\mathcal{D}$, we have:
\begin{align}
&\fnorm{(\U\V\trans - \M)\V}^2 \nn \\
=&\fnorm{\proj_{\hat{\U}}(\U\V\trans - \M)\V}^2 + \fnorm{\proj_{\hat{\U}_\perp}(\U\V\trans - \M)\V}^2 \nn \\
\ge &\sigma_k(\V)^2\fnorm{\proj_{\hat{\U}}(\U\V\trans - \M)\proj_{\hat{\V}}}^2
+ \fnorm{\proj_{\hat{\U}_\perp} \M\V}^2 \nn \\
\ge &\sigma_k(\V)^2\fnorm{\proj_{\hat{\U}}(\U\V\trans - \M)\proj_{\hat{\V}}}^2
+ \sigmin(\Y\trans \V)^2\fnorm{\proj_{\hat{\U}_\perp} \X\mSigma}^2 \nn \\
= &\sigma_k(\V)^2\fnorm{\proj_{\hat{\U}}(\U\V\trans - \M)\proj_{\hat{\V}}}^2
+ \sigmin(\Y\trans \V)^2\fnorm{\proj_{\hat{\U}_\perp} \M}^2 \nn
\end{align}

Therefore, by symmetry, we have:
\begin{align}
\|\nabla f(\U, \V)\|^2_F
= & 4(\fnorm{(\U\V\trans - \M)\V}^2  + \fnorm{(\U\V\trans - \M)\trans\U}^2)  \nn \\
\ge & 4\gamma^2 (2\fnorm{\proj_{\hat{\U}}(\U\V\trans - \M)\proj_{\hat{\V}}}^2
+ \fnorm{\proj_{\hat{\U}_\perp} \M}^2 + \fnorm{ \M\proj_{\hat{\V}_\perp}}^2) \nn\\
\ge & 4\gamma^2 \fnorm{\U\V\trans - \M}^2 \nn
\end{align}
\end{proof}

Next, we show as long as we are in some Frobenious ball around optimum, then we have good spectral property over $\U, \V$ which guarantees the preconditions for Lemma \ref{lem:smoothness_g} and Lemma \ref{lem:pseudostronglyconvex_g}.
\begin{lemma}\label{lem:localguarantee_g}
Within the region $\mathcal{D} = \{(\U, \V) | \norm{\M - \U\V\trans}_F \le \frac{1}{10}\sigma_{k}(\M)\}$, and for $\U = \W_U \D^{\frac{1}{2}}, \V = \W_V \D^{\frac{1}{2}}$ where $\W_U\D\W_V = \text{SVD}(\U\V\trans)$,
we have:
\begin{align}
&\norm{\U} \le \sqrt{2\norm{\M}}, \quad\quad \sigmin(\X\trans \U) \ge \sqrt{\sigma_k(\M)/2} \nn \\
&\norm{\V} \le \sqrt{2\norm{\M}}, \quad\quad \sigmin(\Y\trans \V) \ge \sqrt{\sigma_k(\M)/2}
\nn
\end{align}
\end{lemma}

\begin{proof}
For spectral norm of $\U$, we have:
\begin{equation*}
\norm{\U}^2 = \norm{\D} = \norm{\U\V\trans} \le \norm{\M} + \norm{\M - \U\V\trans}
\le \norm{\M} + \norm{\M - \U\V\trans}_F \le 2 \norm{\M}
\end{equation*}
For the minimum singular value of $\U\trans \U$, we have:
\begin{align}
\sigmin(\U\trans\U) = &   \sigma_k( \D ) = \sigma_k( \U\V \trans ) \ge \sigma_k (\M) - \norm{\M - \U \V\trans} \nn \\
\ge & \sigma_k (\M) - \norm{\M - \U \U\trans}_F
\ge \frac{9}{10} \sigma_k(\M) \nn
\end{align}
By symmetry, the same holds for $\V$.
On the other hand, we have:
\begin{align}
\frac{1}{10}\sigma_k(\M) \ge & \norm{\M - \U\V \trans}_F
\ge \norm{\proj_{\X_\perp} (\M - \U\V\trans)}_F
=\norm{\proj_{\X_\perp} \U\V\trans}_F
=\norm{\proj_{\X_\perp} \W_U \D}_F \nn \\
\ge & \norm{\proj_{\X_\perp} \W_U \D}\ge \frac{9}{10} \sigma_k(\M) \norm{\X_\perp\W_U}
\nn
\end{align}
Let the principal angle between $\X$ and $\W_U$ to be $\theta$.
This gives $\sin^2 \theta = \norm{\X_\perp\trans \W_U}^2 \le \frac{1}{9}$.
Thus $\cos^2 \theta = \sigmin^2(\X\trans \W_U) \ge \frac{8}{9}$.
Therefore:
\begin{equation*}
\sigmin^2(\X \trans \U) \ge \sigmin^2(\X\trans \W_U)\sigmin(\U\trans\U) \ge \sigma_k(\M)/2
\end{equation*}

\end{proof}

\subsection{Proof of Theorem \ref{thm:localcontrol_Asym}}
Now, we are ready for our key theorem. By Lemma \ref{lem:smoothness_g}, Lemma \ref{lem:pseudostronglyconvex_g}, and Lemma \ref{lem:localguarantee_g}, we already know the function has good property locally in the region $\mathcal{D} = \{(\U, \V) | \norm{\M - \U\V\trans}_F \le \frac{1}{10}\sigma_{k}(\M)\}$ which alludes linear convergence. Similar to the symmetric PSD case, the work remains is to prove that once we initialize inside this region, our algorithm will guarantee $\U, \V$ never leave this region with high probability even with relatively large stepsize. Again, we also need to control the incoherence of $\U_t, \V_t$ over all iterates additionally to achieve tight sample complexity and near optimal runtime.

Following is our formal proof.


\begin{proof}[Proof of Theorem \ref{thm:localcontrol_Asym}]
For simplicity of notation, we assume $d = d_1 = d_2$, and do not distinguish $d_1$ and $d_2$. However, it is easy to check our proof never use the property $\M$ is square matrix. The proof easily extends to $d_1\neq d_2$ case by replacing $d$ in the proof with suitable $d_1, d_2$.

Define event $\fE_t = \{\forall \tau\le t, f(\U_\tau, \V_\tau) \le (1-\frac{\eta}{2\cn})^t(\frac{1}{10 \cn})^2, \max_i g_i(\U_\tau, \V_\tau) \le \frac{100\mu k \cn^2}{d}, \max_j h_j(\U_\tau, \V_\tau) \le \frac{100\mu k \cn^2}{d}\}$.
Theorem \ref{thm:localcontrol_Asym} is equivalent to prove event $\fE_T$ happens with high probability. The proof achieves this by contructing two supermartingales for $f(\U_t, \V_t)1_{\fE_t}$, $g_i(\U_t, \V_t)1_{\fE_t}$ and $h_i(\U_t, \V_t)1_{\fE_t}$ (where $1_{(\cdot)}$ denote indicator function), applies concentration argument.

The proofs also follow similar structure as symmetric PSD case:
\begin{enumerate}
 \item The constructions of supermartingales
 \item Their probability 1 bound and variance bound in order to apply Azuma-Bernstein inequality
 \item Final combination of concentration results to conclude the proof
\end{enumerate}

Then let filtration $\fF_t = \sigma\{SG(\U_0, \V_0), \cdots, SG(\U_{t-1}, \V_{t-1})\}$ where $\sigma\{\cdot\}$ denotes the sigma field. 
Also let event , note $\fE_t \subset \fF_t$. Also $\fE_{t+1} \subset \fE_t$, and thus $1_{\fE_{t+1}} \le 1_{\fE_{t}}$. 

By Lemma \ref{lem:localguarantee_g}, we immediately know that conditioned on $\fE_t$, we have $\norm{\U_t} \le \sqrt{2}$, $\norm{\V_t} \le \sqrt{2}$,  $\sigma_{\min}(\X\trans\U_t) \ge 1/\sqrt{2\cn}$, $\sigma_{\min}(\Y\trans\V_t) \ge 1/\sqrt{2\cn}$. We will use this fact throughout the proof.

For simplicity, when it's clear from the context, we denote:
\begin{align}
\begin{pmatrix}
\Delta_\U \\ \Delta_\V
\end{pmatrix}
= -\eta SG(\tU, \tV) = \begin{pmatrix}
\U_{t+1} \\ \V_{t+1}
\end{pmatrix}
-\begin{pmatrix}
\tU \\ \tV
\end{pmatrix} \nn
\end{align}

\paragraph{Construction of supermartingale $G$:}
First, since potential function $g_l(\U, \V)$ is forth-order polynomial, we can expand:
\begin{align}
g_l(\tilde{\U}_{t+1}, \tilde{\V}_{t+1}) &=  g_l(\U_{t+1}, \V_{t+1})
= g_l(\tU+\Delta_\U, \tV + \Delta_\V) \nn \\
&= \e_l\trans (\tU+\Delta_\U)( \tV + \Delta_\V)\trans (\tV + \Delta_\V)(\tU+\Delta_\U)\trans \e_l
\nn \\
&= g_l(\tU, \tV) + 2\e_l\trans \Delta_\U \tV\trans \tV\tU\trans \e_l +
2\e_l\trans\tU \tV\trans \Delta_\V \tU \e_l + R_2 \nn \\
&= g_l(\tU, \tV) + R_1 \nn
\end{align}
Where we denote $R_1$ as the sum of first order terms and higher order terms (all second/third/forth order terms), and $R_2$ as the sum of second order terms and higher order terms.

We now give a proposition about properties of $R_1$ and $R_2$ which involves a lot calculation, and postpone its proof in the end of this section.

\begin{proposition}\label{prop:G}
With above notations, we have following inequalities hold true.
\begin{align*}
&\E [R_2 1_{\fE_t}|\fF_t] \le \eta^2 O(\mu^2 k^2 \cn^4)1_{\fE_t} \nn \\
&|R_1|1_{\fE_{t}} \le \eta O(\mu^2k^2 \cn^5)1_{\fE_{t}}  \quad\quad \text{w.p~} 1\nn \\
&\E [R_1^2 1_{\fE_{t}} |\fF_t]\le \eta^2 O(\frac{\mu^3 k^3 \cn^6}{d})1_{\fE_{t}} 
\end{align*}
\end{proposition}

Then by taking conditional expectation, we have:
\begin{align}
&\E [g_l(\tilde{\U}_{t+1}, \tilde{\V}_{t+1})1_{\fE_t}|\fF_t] 
= \E [g_l(\U_{t+1}, \V_{t+1})1_{\fE_t}|\fF_t] \nn \\
\le& \E [g_l(\tU, \tV) +2\e_l\trans \Delta_\U \tV\trans \tV\tU\trans \e_l +
2\e_l\trans\tU \tV\trans \Delta_\V \tU \e_l + R_2]1_{\fE_t} \nn
\end{align}

The first order term can be calculated as:
\begin{align}
&[- \E 2\e_l\trans \Delta_\U \tV\trans \tV\tU\trans \e_l +
2\e_l\trans\tU \tV\trans \Delta_\V \tU \e_l ]1_{\fE_t}  \nn\\
=&[- 4 \e_l\trans (\tU\tV\trans - \M)\tV \tV\trans \tV \tU\trans\e_l
- 4 \e_l\trans \tU \tU\trans(\tU\tV\trans- \M\trans)\tV\tU\trans \e_l ]1_{\fE_t} \nn\\
=&[-4  \e_l\trans \tU\tV\trans\tV \tV\trans \tV \tU\trans\e_l
+ 4\e_l\trans \M\tV \tV\trans \tV \tU\trans\e_l
- 4 \e_l\trans \tU \tU\trans(\tU\tV\trans- \M\trans)\tV\tU\trans \e_l ]1_{\fE_t} \nn\\
\le& [-4[ \sigma_{\min}(\tV\trans\tV)\norm{\e_l\trans\tU\tV\trans}^2
+ \norm{\e_l\trans\tU\tV\trans}\norm{\tV\tV\trans}\norm{\e_l\trans\M}  \nn\\
&+ \norm{\e_l\trans \tU \tU\trans}\norm{\tU\tV\trans- \M\trans}_F \norm{\e_l\trans \tU \tV\trans}] ]1_{\fE_t}
\nn\\
\le& [-\frac{2}{\cn}g_l(\tU, \tV) + 80\frac{\mu k \cn}{d} +\frac{4}{10\cn}g_l(\tU, \tV)]1_{\fE_t} \nn\\
\le& [- \frac{1}{\cn} g_l(\tU, \tV) + 80\frac{\mu k \cn}{d}]1_{\fE_t} \nn
\end{align}
In second last inequality, we use key observation:
\begin{equation*}
\norm{\e_k\trans \tU\tV\trans} = \norm{\e_k \trans \W_\U\D\W_\V\trans}
= \norm{\e_k \trans \W_\U\D\W_\U\trans}
=\norm{\e_k\trans \tU\tU\trans} 
\end{equation*}
By Proposition \ref{prop:G}, we know $\E [R_2 1_{\fE_t}|\fF_t] \le \eta^2 O(\mu^2 k^2 \cn^4)1_{\fE_t}$.
Combine both facts and recall $\eta < \frac{c}{\mu dk \cn^3 \log d}$, we have:
\begin{align}
\E [g_i(\tilde{\U}_{t+1}, \tilde{\V}_{t+1})1_{\fE_t}|\fF_t] 
&\le [(1- \frac{\eta}{\cn})  g_i(\tU, \tV) + \frac{80\eta \mu k \cn}{d} +  O(\eta^2\mu^2 k^2 \cn^4)]1_{\fE_t} \nn \\
&\le [(1- \frac{\eta}{\cn})  g_i(\tU, \tV) + \frac{90\eta \mu k \cn}{d} ]1_{\fE_t} \nn
\end{align}
The last inequality is achieved by choosing $c$ small enough. 

Let $G_{it} = (1-\frac{\eta}{\cn})^{-t}(g_i(\tU, \tV)1_{\fE_{t-1}} - 90\frac{\mu k\cn^2}{d})$. This gives:
\begin{equation*}
\E [G_{i(t+1)}|\fF_{t}]
\le (1-\frac{\eta}{\cn})^{-t}(g_i(\tU, \tV)1_{\fE_{t}} - 90\frac{\mu k\cn^2}{d})
\le G_{it}
\end{equation*}
The right inequality is true since $1_{\fE_t} \leq 1_{\fE_{t-1}}$. This implies
$G_{it}$ is supermartingale.

\paragraph{Probability 1 bound for $G$:}
We also know:
\begin{align}
G_{i(t+1)} - \E [G_{i(t+1)} | \fF_t] = & (1-\frac{\eta}{\cn})^{-(t+1)}
[R_1 - \E R_1]1_{\fE_{t}} \nn
\end{align}
By Proposition \ref{prop:G}, we know with probability 1 that $|R_1|1_{\fE_{t}} \le \eta O(\mu^2k^2 \cn^5)1_{\fE_{t}}$. This gives with probability 1:
\begin{equation}\label{p1_G_g}
|G_{it} - \E [G_{it}|\fF_{t-1}]| \le (1-\frac{\eta}{\cn})^{-t} \eta O(\mu^2k^2 \cn^5)1_{\fE_{t-1}}
\end{equation}

\paragraph{Variance bound for $G$:}
We also know
\begin{align}
\Var(G_{i(t+1)}| \fF_t) = (1-\frac{\eta}{\cn})^{-2(t+1)}[\E R_1^21_{\fE_{t}} - (\E R_1 1_{\fE_{t}})^2]
\le \E [R_1^2 1_{\fE_{t}}| \fF_t] \nn
\end{align}
By Proposition \ref{prop:G}, we know that $\E [R_1^2 1_{\fE_{t}}| \fF_t] \le \eta^2 O(\frac{\mu^3 k^3 \cn^6}{d})1_{\fE_{t}}$. This gives:
\begin{equation}\label{var_G_g}
\Var(G_{it} | \fF_{t-1}) \le (1-\frac{\eta}{\cn})^{-2t} \eta^2 O(\frac{\mu^3 k^3 \cn^6}{d})1_{\fE_{t-1}}
\end{equation}

\paragraph{Bernstein's inequality for $G$:}
Let $\sigma^2 = \sum_{\tau=1}^t \Var(G_{i\tau} | \fF_{\tau-1})$, and 
$R$ satisfies, with probability 1 that $|G_{i\tau} - \E [G_{i\tau}|\fF_{\tau-1}]| \le R, ~\tau = 1, \cdots, t$. Then 
By standard Bernstein concentration inequality, we know:
\begin{equation*}
P(G_{it} \ge G_{i0} + s ) \le \exp(\frac{s^2/2}{\sigma^2 + Rs / 3})
\end{equation*}

Since $G_{i0} = g_i(\tilde{\U}_0, \tilde{\V}_0) - 90 \frac{\mu k \cn^2}{d}$, let $s' = O(1)(1- \frac{\eta}{\cn})^t [\sqrt{\sigma^2 \log d} + R\log d]$, 
we know 
\begin{equation*}
P\left(g_i(\tU, \tV)1_{\fE_{t-1}} \ge 90 \frac{\mu k \cn^2}{d}
+ (1- \frac{\eta}{\cn})^t(g_i(\tilde{\U}_0, \tilde{\V}_0) - 90 \frac{\mu k \cn^2}{d}) + s'\right) \le \frac{1}{3d^{11}}
\end{equation*}

By Eq.\eqref{p1_G_g}, we know $R = (1-\frac{\eta}{\cn})^{-t} \eta O(\mu^2k^2 \cn^5)$ satisfies that $|G_{i\tau} - \E [G_{i\tau}|\fF_{\tau-1}]| \le R, ~\tau = 1, \cdots, t$. Also by Eq. \eqref{var_G_g}, we have:
\begin{align}
(1- \frac{\eta}{\cn})^t \sqrt{\sigma^2 \log d}
\le \eta O(\sqrt{\frac{\mu^3 k^3 \cn^6 \log d}{d}})\sqrt{\sum_{\tau=1}^t(1- \frac{\eta}{\cn})^{2t-2\tau} }
\le \sqrt{\eta} O(\sqrt{\frac{\mu^3 k^3 \cn^7 \log d}{d}}) \nn
\end{align}
by $\eta < \frac{c}{\mu dk \cn^3 \log d}$ and choosing $c$ to be small enough, we have:
\begin{equation*}
s' = \sqrt{\eta} O(\sqrt{\frac{\mu^3 k^3 \cn^7 \log d}{d}}) + \eta O(\mu^2k^2 \cn^5 \log d)
\le 10\frac{\mu k \cn^2}{d} 
\end{equation*}
Since initialization gives $\max_i g_i(\U_0, \V_0) \le \frac{10\mu k \cn^2}{d}$, therefore:
\begin{equation*}
P(g_i(\tU, \tV)1_{\fE_{t-1}} \ge 100\frac{\mu k \cn^2}{d}) \le \frac{1}{3d^{11}}
\end{equation*}
That is equivalent to:
\begin{equation}
P(\fE_{t-1} \cap \{g_i(\tU, \tV)\ge 100\frac{\mu k \cn^2}{d}\}) \le \frac{1}{3d^{11}}\label{concen_G_g}
\end{equation}

By symmetry, we can also have corresponding result for $h_j(\tU, \tV)$.

~
\paragraph{Construction of supermartingale F:}
Similarly, we also need to construct a martingale for $f(\tU, \tV)$. Again,
we can write $f$ as forth order polynomial:
\begin{align}
f(\tilde{\U}_{t+1}, \tilde{\V}_{t+1}) &= f(\U_{t+1}, \V_{t+1})
= f(\tU + \Delta_\U, \tV + \Delta_\V) \nn \\
&= \tr\left([(\tU + \Delta_\U)(\tV + \Delta_\V) - \M ][(\tU + \Delta_\U)(\tV + \Delta_\V) - \M ]\trans\right) \nn  \\
&= f(\tU, \tV) + 2 \tr( \Delta_\U\tV \trans (\tU\tV\trans - \M)\trans)
+2 \tr(\Delta_\V \tU\trans (\tU \tV\trans - \M)) + Q_2 \nn \\
&= f(\tU, \tV) + Q_1 \nn
\end{align}

Where we denote $Q_1$ as the sum of first order terms and higher order terms (all second/third/forth order terms), and $Q_2$ as the sum of second order terms and higher order terms.

We also now give a proposition about properties of $Q_1$ and $Q_2$ which involves a lot calculation, and postpone its proof in the end of this section.

\begin{proposition}\label{prop:F}
With above notations, we have following inequalities hold true.
\begin{align}
&\E[Q_2 1_{\fE_t} |\fF_t] \le \eta^2 O(\mu dk \cn^2)f(\tU, \tV)1_{\fE_t} \nn\\
&|Q_1|1_{\fE_t} \le \eta O(\mu d k \cn^{3}) f(\tU, \tV)1_{\fE_{t}}  \quad\quad \text{w.p~} 1 \nn\\
&\E [Q_1^2 1_{\fE_{t}}|\fF_t] \le \eta^2 O(\mu d k \cn^2) f^2(\tU, \tV)1_{\fE_{t}} \nn
\end{align}
\end{proposition}

Then by Proposition \ref{prop:F}, we know $\E[Q_2 1_{\fE_t} |\fF_t] \le \eta^2 O(\mu dk \cn^2)f(\tU, \tV)1_{\fE_t}$.
By taking conditional expectation, we have:
\begin{align}
&\E [f(\U_{t+1})1_{\fE_t} |\fF_t] \nn\\
\le& [f(\tU, \tV) - \E \langle \nabla f(\tU, \tV), \eta SG(\U_t)\rangle + \E Q_2]1_{\fE_t} \nn \\
= & [f(\tU, \tV) - \eta \norm{\nabla f(\tU, \tV)}_F^2 + \E Q_2]1_{\fE_t} \nn \\
\le& [(1-\frac{2\eta}{\cn}) f(\tU, \tV) + \eta^2 O(\mu dk \cn^2)f(\tU, \tV)]1_{\fE_t} \nn\\
\le& (1-\frac{\eta}{\cn}) f(\tU, \tV)1_{\fE_t} \nn
\end{align}

Let $F_t =  (1-\frac{\eta}{\cn})^{-t}f(\tU, \tV)1_{\fE_{t-1}}$, we know 
$F_t$ is also a supermartingale.

\paragraph{Probability 1 bound:}
We also know
\begin{equation*}
F_{t+1} - \E [F_{t+1}|\fF_t] = (1-\frac{\eta}{\cn})^{-(t+1)}[Q_1 - \E Q_1]1_{\fE_t}
\end{equation*}
By Proposition \ref{prop:F}, we know with probability 1 that 
$|Q_1|1_{\fE_t} \le \eta O(\mu d k \cn^{3}) f(\U_{t}, \V_t)1_{\fE_{t}}$. This gives with probability 1:
\begin{equation}
|F_{t} - \E F_{t}| \le (1-\frac{\eta}{\cn})^{-t} \eta O(\mu d k \cn^{3}) f(\U_{t-1})1_{\fE_{t-1}}
\le (1-\frac{\eta}{\cn})^{-t}(1-\frac{\eta}{2\cn})^t \eta O(\mu d k \cn)1_{\fE_{t-1}}\label{p1_F_g}
\end{equation}

\paragraph{Variance bound:}
We also know 
\begin{align}
\Var(F_{t+1}| \fF_t) = (1-\frac{\eta}{\cn})^{-2(t+1)}[\E Q_1^21_{\fE_{t}} - (\E Q_1 1_{\fE_{t}})^2]
\le (1-\frac{\eta}{\cn})^{-2(t+1)} \E [Q_1^2 1_{\fE_{t}}|\fF_t] \nn
\end{align}
By Proposition \ref{prop:F}, we know that $\E [Q_1^2 1_{\fE_{t}}|\fF_t] \le \eta^2 O(\mu d k \cn^2) f^2(\U_{t}, \V_t)1_{\fE_{t}}$. This gives:
\begin{equation}
\Var(F_t|\fF_{t-1}) \le (1-\frac{\eta}{\cn})^{-2t} \eta^2 O(\mu d k \cn^2) f^2(\U_{t-1})1_{\fE_{t-1}}
\le (1-\frac{\eta}{\cn})^{-2t} (1-\frac{\eta}{2\cn})^{2t}\eta^2 O(\frac{\mu d k}{\cn^2}) 1_{\fE_{t-1}} \label{var_F_g}
\end{equation}

\paragraph{Bernstein's inequality:} Let $\sigma^2 = \sum_{\tau=1}^t \Var(F_{\tau} | \fF_{\tau-1})$, and 
$R$ satisfies, with probability 1 that $|F_{\tau} - \E [F_{\tau}||\fF_{\tau-1}]| \le R, ~\tau = 1, \cdots, t$. Then 
By standard Bernstein concentration inequality, we know:
\begin{equation*}
P(F_{t} \ge F_{0} + s ) \le \exp(\frac{s^2/2}{\sigma^2 + Rs / 3})
\end{equation*}
Let $s' = O(1)(1- \frac{\eta}{\cn})^t [\sqrt{\sigma^2 \log d} + R\log d]$, this gives:
\begin{equation*}
P(f(\tU, \tV)1_{\fE_{t-1}} \ge (1- \frac{\eta}{\cn})^t f(\U_0)+ s') \le \frac{1}{3d^{10}}
\end{equation*}
By Eq.\eqref{p1_F_g}, we know $R = (1-\frac{\eta}{\cn})^{-t} (1-\frac{\eta}{2\cn})^t \eta O(\mu d k \cn)$ satisfies that $|F_{\tau} - \E [F_{\tau}|\fF_{\tau-1}]| \le R, ~\tau = 1, \cdots, t$. Also by Eq. \eqref{var_F_g}, we have:
\begin{align}
&(1- \frac{\eta}{\cn})^t \sqrt{\sigma^2 \log d}
\le \eta O(\sqrt{\frac{\mu d k \log d}{\cn^2}})\sqrt{\sum_{\tau=1}^t(1- \frac{\eta}{\cn})^{2t-2\tau}(1- \frac{\eta}{2\cn})^{2\tau}  } \nn \\
\le & (1-\frac{\eta}{2\cn})^t
\eta O(\sqrt{\frac{\mu d k \log d}{\cn^2}})\sqrt{\sum_{\tau=1}^t(1- \frac{\eta}{\cn})^{2t-2\tau}(1- \frac{\eta}{2\cn})^{2\tau-2t}  } 
\le (1-\frac{\eta}{2\cn})^t \sqrt{\eta} O(\sqrt{\frac{\mu d k \log d}{\cn}}) \nn
\end{align}
by $\eta < \frac{c}{\mu dk \cn^3 \log d}$ and choosing $c$ to be small enough, we have:
\begin{equation*}
s' = (1-\frac{\eta}{2\cn})^t [\sqrt{\eta} O(\sqrt{\frac{\mu d k \cn \log d}{\cn}}) + \eta O(\mu d k \cn)]
\le (1-\frac{\eta}{2\cn})^t(\frac{1}{20\cn})^2
\end{equation*}
Since $F_{0} = f(\U_0) \le (\frac{1}{20\cn})^2$, therefore:
\begin{equation*}
P(f(\tU, \tV)1_{\fE_{t-1}} \ge (1-\frac{\eta}{2\cn})^t(\frac{1}{10\cn})^2) \le \frac{1}{3d^{10}}
\end{equation*}
That is equivalent to:
\begin{equation}
P(\fE_{t-1} \cap \{f(\tU, \tV)\ge (1-\frac{\eta}{2\cn})^t(\frac{1}{10\cn})^2\}) \le \frac{1}{3d^{10}} \label{concen_F_g}
\end{equation}

\paragraph{Probability for event $\fE_T$:} 

Finally, combining the concentration result for martingale $G$ (Eq.\eqref{concen_G_g}) and martingale $F$ (Eq.\eqref{concen_F_g}), we conclude:
\begin{align}
&P(\fE_{t-1} \cap \bar{\fE_t}) \nn \\
=  &P\left[\fE_{t-1} \cap \left([\cup_i \{g_i(\U_t, \V_t) \ge 100 \frac{\mu k \cn^2}{d} \}]\right.\right. \nn \\
&\cup [\cup_j \{h_j(\U_t, \V_t) \ge 100 \frac{\mu k \cn^2}{d} \}]\cup \left.\left.
\{f(\U_t)\ge (1-\frac{\eta}{2\cn})^t(\frac{1}{10\cn})^2\}\right)\right]
\nn \\
\le & 2\sum_{i=1}^d P(\fE_{t-1} \cap \{g_i(\U_t, \V_t)\ge 100\frac{\mu k \cn^2}{d}\})
+ P(\fE_{t-1} \cap \{f(\U_t, \V_t)\ge (1-\frac{\eta}{2\cn})^t(\frac{1}{10\cn})^2\}) \nn \\
\le & \frac{1}{d^{10}} \nn
\end{align}
Since
\begin{equation*}
P(\bar{\fE_T}) = \sum_{t=1}^T P(\fE_{t-1} \cap \bar{\fE_t}) \le \frac{T}{d^{10}}
\end{equation*}
We finishes the proof.
\end{proof}



Finally we give proof for Proposition \ref{prop:G} and Proposition \ref{prop:F}. The proof mostly consistsof expanding every term and careful calculations.

\begin{proof}[Proof of Proposition \ref{prop:G}]
For simplicity of notation, we hide the term $1_{\fE_t}$ in all following equations. Reader should always think every term in this proof multiplied by $1_{\fE_t}$.
Recall that:
\begin{align}
SG(\U, \V) &= 
2d^2(\U\V\trans - \M)_{ij}\begin{pmatrix}
\e_i\e_j\trans \V \\
\e_j\e_i\trans \U
\end{pmatrix} \nn \\
\begin{pmatrix}
\Delta_\U \\ \Delta_\V
\end{pmatrix}
=& -\eta SG(\tU, \tV) = \begin{pmatrix}
\U_{t+1} \\ \V_{t+1}
\end{pmatrix}
-\begin{pmatrix}
\tU \\ \tV
\end{pmatrix} \nn
\end{align}
We first prove first three inequality. Recall that:
\begin{align}
g_l(\tilde{\U}_{t+1}, \tilde{\V}_{t+1}) &=  g_l(\U_{t+1}, \V_{t+1})
= g_l(\tU+\Delta_\U, \tV + \Delta_\V) \nn \\
&= \e_l\trans (\tU+\Delta_\U)( \tV + \Delta_\V)\trans (\tV + \Delta_\V)(\tU+\Delta_\U)\trans \e_l
\nn \\
&= g_l(\tU, \tV) + 2\e_l\trans \Delta_\U \tV\trans \tV\tU\trans \e_l +
2\e_l\trans\tU \tV\trans \Delta_\V \tU \e_l + R_2 \nn \\
&= g_l(\tU, \tV) + R_1 \nn
\end{align}
By expanding the polynomial, we can write out the first order term:
\begin{align}
R_1 - R_2 =& 2\e_l\trans \Delta_\U \tV\trans \tV\tU\trans \e_l +
2\e_l\trans\tU \tV\trans \Delta_\V \tU \e_l \nn \\
=& -4\eta d^2 (\tU\tV\trans - \M)_{ij}\left(
\delta_{il} \e_j\trans \tV \tV\trans \tV\tU\trans \e_l +
(\tU \tV\trans)_{lj} (\tU\tU\trans)_{il}
\right) \nn
\end{align}
Second order term:
\begin{align}
&R_2-R_3 \nn\\
=& \e_l\trans \Delta_\U \tV\trans \tV\Delta_\U\trans \e_l
+ \e_l\trans \tU \Delta_\V\trans \Delta_\V \tU\trans \e_l 
+ 2\e_l\trans \Delta_\U \tV\trans \Delta_\V \tU\trans  \e_l
+ 2\e_l\trans \Delta_\U \Delta_\V\trans \tV\tU\trans  \e_l \nn \\
= & 4\eta^2 d^4(\tU\tV\trans - \M)^2_{ij} \nn \\
&\cdot
\left(
\delta_{il} \norm{\e_j\trans \tV \tV\trans}^2
+ (\tU\tU\trans)_{li}^2
+ 2\delta_{il} (\tV\tV\trans)_{jj}(\tU\tU\trans)_{ii}
+ 2\delta_{il} (\tU\tV\trans)^2_{ij}
\right) \nn
\end{align}
Third order term:
\begin{align}
R_4-R_3 = &
2\e_l\trans \Delta_\U \tV\trans \Delta_\V \Delta_\U\trans \e_l
+ 2\e_l\trans \Delta_\U \Delta_\V\trans \Delta_\V \tU\trans \e_l  \nn\\
=& -16 \eta^3 d^6(\tU\tV\trans - \M)^3_{ij} \delta_{il}
\left(
(\tV\tV)_{jj}(\tU\tV)_{ij} + (\tU\tV)_{ij} (\tU\tU)_{ii}
\right) \nn
\end{align}
Fourth order term:
\begin{align}
R_4 = \e_l\trans \Delta_\U \Delta_\V\trans \Delta_\V \Delta_\U\trans \e_l
= 16\eta^4d^8(\tU\tV\trans - \M)^4_{ij} \delta_{il}(\tU\tV)_{ij}^2 \nn
\end{align}
For the ease of proof, we denote $\chi = \frac{\mu k \cn^2}{d}$, then we know
conditioned on event $\fE_t$, we have: $\max_i \norm{\e_i\trans \tU \tV\trans}^2 \le O(\chi)$, 
and $\max_j \norm{\e_j\trans \tV \tU\trans}^2 \le O(\chi)$.
Some key inequality we need to use in the proof are listed here:
\begin{equation}
\norm{\e_l\trans \tU\tV\trans} 
=\norm{\e_l\trans \tU\tU\trans}
\quad\text{and}\quad
 \norm{\e_l\trans \tV\tU\trans} 
=\norm{\e_l\trans \tV\tV\trans}
\label{aux_1}
\end{equation}
and 
\begin{equation}
|(\tU\tV)_{ij}| \le \norm{\e_i\trans \tU \tV\trans} \le O(\sqrt{\chi})
\label{aux_2}
\end{equation}
The same also holds true for: 
\begin{equation}
|(\tU\tU\trans)_{ii}| \le O(\sqrt{\chi}) \quad \text{and} \quad |(\tV\tV\trans)_{jj}| \le O(\sqrt{\chi})\label{aux_3}
\end{equation}
Another fact we frequently used is:
\begin{equation*}
\frac{1}{2} \norm{\e_i\trans \U \V\trans}^2 \le \norm{\e_i\trans \U}^2 
\le 2\cn \norm{\e_i\trans \U \V\trans}^2
\end{equation*}
This gives:
\begin{align}
\norm{\tU\tV\trans - \M}_\infty
\le \max_k\norm{\e_k \tU}\max_k\norm{\e_k \tV} + \max_k\norm{\e_k \X}\max_k\norm{\e_k \Y} \norm{\S} \le O(\chi \cn) \nn
\end{align}
and recall we choose $\eta < \frac{c}{\mu dk \cn^3 \log d}$, where $c$ is some universal constant, then we have:
\begin{equation}
\eta d^2\norm{\tU\tV\trans - \M}_\infty 
= O(\eta d^2 \chi \cn) \le O(1) \label{aux_4}
\end{equation}

With equation \eqref{aux_1}, \eqref{aux_2}, \eqref{aux_3}, \eqref{aux_4}, now we are ready to prove Lemma.

\paragraph{For the first inequality $\E [R_2 1_{\fE_t}|\fF_t] \le \eta^2 O(\mu^2 k^2 \cn^4)1_{\fE_t}$:}
\begin{align}
\E [R_2 1_{\fE_t}|\fF_t]
\le \E [|R_2-R_3| 1_{\fE_t}|\fF_t] + \E [|R_3-R_4| 1_{\fE_t}|\fF_t] + \E [|R_4| 1_{\fE_t}|\fF_t] \nn
\end{align}
For each term, we can bound as:
\begin{align}
&\E [|R_2-R_3| 1_{\fE_t}|\fF_t] 
\le  \eta^2 O( d^2) \sum_{ij}(\tU\tV\trans - \M)^2_{ij}
\left(\delta_{il} O(\chi)
+ (\tU\tU\trans)_{li}^2
\right) \nn \\
\le & \eta^2 O(d^2)\max_{l'}\norm{\e_{l'}\trans(\tU\tV\trans - \M)}^2
\sum_i\left(\delta_{il} O(\chi)
+ (\tU\tU\trans)_{li}^2 \right)  \le \eta^2 O(d^2\chi^2) \nn \\
&\E [|R_3-R_4| 1_{\fE_t}|\fF_t]
\le 
\eta^3 O(d^4) \sum_{ij}|\tU\tV\trans - \M|^3_{ij} \delta_{il}O(\chi)  \nn \\
\le & \eta^3 O(d^4) \norm{\tU\tV\trans - \M}_\infty \norm{\e_l\trans(\tU\tV\trans - \M)}^2 O(\chi)  \le \eta^2 O(d^2\chi^2) \nn  \\
&\E [|R_4| 1_{\fE_t}|\fF_t] \le 
\eta^4 O(d^6) \sum_{ij}(\tU\tV\trans - \M)^4_{ij} \delta_{il}O(\chi) \nn \\
\le & \eta^4 O(d^6) \norm{\tU\tV\trans - \M}^2_\infty \norm{\e_l\trans(\tU\tV\trans - \M)}^2 O(\chi)\le \eta^2 O(d^2\chi^2) \nn
\end{align}
This gives in sum that 
\begin{equation*}
\E [R_2 1_{\fE_t}|\fF_t] \le \eta^2 O(d^2\chi^2)1_{\fE_{t}}  = \eta^2 O(\mu d k \cn^2) f^2(\U_{t})1_{\fE_{t}}
\end{equation*}

\paragraph{For the second inequality $|R_1|1_{\fE_{t}} \le \eta O(\mu^2k^2 \cn^5)1_{\fE_{t}}  \text{~w.p~} 1$:}
\begin{align}
|R_1|1_{\fE_{t}} \le |R_1-R_2|1_{\fE_{t}}  + |R_2- R_3|1_{\fE_{t}}  + |R_3-R_4|1_{\fE_{t}}  + |R_4|1_{\fE_{t}}  \nn
\end{align}
For each term, we can bound as:
\begin{align*}
|R_1-R_2|1_{\fE_{t}} \le &\eta O(d^2) \norm{\tU\tV\trans - \M}_{\infty}O(\chi)
\le \eta O(d^2\chi^2 \cn) \\
|R_2- R_3|1_{\fE_{t}} \le &\eta^2 O(d^4) \norm{\tU\tV\trans - \M}^2_{\infty}O(\chi)
\le \eta O(d^2\chi^2 \cn) \\
|R_3- R_4|1_{\fE_{t}} \le &\eta^3 O(d^6) \norm{\tU\tV\trans - \M}^3_{\infty}O(\chi)
\le \eta O(d^2\chi^2 \cn) \\
|R_4|1_{\fE_{t}} \le &\eta^4 O(d^8) \norm{\tU\tV\trans - \M}^4_{\infty}O(\chi)
\le \eta O(d^2\chi^2 \cn) 
\end{align*}
This gives in sum that, with probability 1:
\begin{equation*}
|R_1|1_{\fE_{t}}  \le \eta O(d^2\chi^2 \cn)1_{\fE_{t}}   = \eta O(\mu^2k^2 \cn^5)1_{\fE_{t}} 
\end{equation*}

\paragraph{For the third inequality $\E [R_1^2 1_{\fE_{t}}|\fF_{t}] \le \eta^2 O(\frac{\mu^3 k^3 \cn^6}{d})1_{\fE_{t}}$:}
\begin{equation*}
\E R_1^2 1_{\fE_{t}}
\le 4 \left[\E (R_1-R_2)^2 1_{\fE_{t}} + \E (R_2-R_3)^2 1_{\fE_{t}} + \E (R_3-R_4)^2 1_{\fE_{t}} + \E R_4^2 1_{\fE_{t}} \right]
\end{equation*}
For each term, we can bound as:
\begin{align*}
&\E (R_1-R_2)^2 1_{\fE_{t}} \le  \eta^2 O(d^2) \sum_{ij} (\tU\tV\trans - \M)^2_{ij}
\left(
\delta_{il} O(\chi^2) +
(\tU \tV\trans)^2_{lj} (\tU\tU\trans)^2_{il}
\right) \nn \\
\le & \eta^2 O(d^2) \max_{l'}\norm{\e_{l'}\trans(\tU\tV\trans - \M)}^2
\sum_i \left(
\delta_{il} O(\chi^2) +
(\tU \tV\trans)^2_{lj} (\tU\tU\trans)^2_{il}
\right) \le \eta^2 O(d^2\chi^3) \nn \\
&\E (R_2-R_3)^2 1_{\fE_t} 
\le  \eta^4 O( d^6) \sum_{ij}(\tU\tV\trans - \M)^4_{ij}
\left(\delta_{il} O(\chi^2)
+ (\tU\tU\trans)_{li}^4
\right) \nn \\
\le & \eta^4 O(d^6)\norm{\tU\tV\trans - \M}_\infty^2\max_{l'}\norm{\e_{l'}\trans(\tU\tV\trans - \M)}^2
\sum_i\left(\delta_{il} O(\chi^2)
+ (\tU\tU\trans)_{li}^4 \right)  \le \eta^2 O(d^2\chi^3) \nn \\
&\E (R_3-R_4)^2 1_{\fE_t}
\le 
\eta^6 O(d^{10}) \sum_{ij}|\tU\tV\trans - \M|^6_{ij} \delta_{il}O(\chi^2)  \nn \\
\le & \eta^6 O(d^{10}) \norm{\tU\tV\trans - \M}^4_\infty \norm{\e_l\trans(\tU\tV\trans - \M)}^2 O(\chi^2)  \le \eta^2 O(d^2\chi^3)  \\
&\E R_4^2 1_{\fE_t} \le 
\eta^8 O(d^{14}) \sum_{ij}(\tU\tV\trans - \M)^8_{ij} \delta_{il}O(\chi^2) \nn \\
\le & \eta^8 O(d^{14}) \norm{\tU\tV\trans - \M}^6_\infty \norm{\e_l\trans(\tU\tV\trans - \M)}^2 O(\chi^2)\le \eta^2 O(d^2\chi^3) 
\end{align*}
This gives in sum that:
\begin{equation*}
\E R_1^2 1_{\fE_{t}} \le \eta^2 O(d^2\chi^3)1_{\fE_{t}} = \eta^2 O(\frac{\mu^3 k^3 \cn^6}{d})1_{\fE_{t}}
\end{equation*}
This finishes the proof.
\end{proof}

\begin{proof}[Proof of Proposition \ref{prop:F}]
Similarly to the proof of Proposition \ref{prop:G},  we hide the term $1_{\fE_t}$ in all following equations. Reader should always think every term in this proof multiplied by $1_{\fE_t}$. Recall that:
\begin{align*}
f(\tilde{\U}_{t+1}, \tilde{\V}_{t+1}) &= f(\U_{t+1}, \V_{t+1})
= f(\tU + \Delta_\U, \tV + \Delta_\V) \nn \\
&= \tr\left([(\tU + \Delta_\U)(\tV + \Delta_\V) - \M ][(\tU + \Delta_\U)(\tV + \Delta_\V) - \M ]\trans\right) \nn  \\
&= f(\tU, \tV) + 2 \tr( \Delta_\U\tV \trans (\tU\tV\trans - \M)\trans)
+2 \tr(\Delta_\V \tU\trans (\tU \tV\trans - \M)) + Q_2 \nn \\
&= f(\tU, \tV) + Q_1
\end{align*}

By expanding the polynomial, we can write out the first order term:
\begin{align*}
Q_1 - Q_2 =& 2 \tr( \Delta_\U\tV \trans (\tU\tV\trans - \M)\trans)
+2 \tr(\Delta_\V \tU\trans (\tU \tV\trans - \M)) \nn \\
=& -4\eta d^2(\U\V\trans - \M)_{ij}\left(
\e_j\trans \tV\tV\trans (\tU\tV\trans - \M)\trans \e_i
+ \e_i\trans \tU\tU\trans (\tU\tV\trans - \M) \e_j
\right)
\end{align*}
The second order term:
\begin{align*}
&Q_2 - Q_3  \nn \\
=& \tr(\Delta_\U \tV\trans \tV\Delta_\U\trans)
+ \tr(\tU \Delta_\V\trans \Delta_\V \tU\trans)
+ 2\tr(\Delta_\U \tV\trans\Delta_\V \tU\trans)
+ 2\tr(\Delta_\U \Delta_V\trans (\tU\tV\trans - \M)\trans) \nn \\
=& 4\eta^2 d^4(\U\V\trans - \M)^2_{ij}\nn \\
&\cdot
\left(
\norm{\e_j\trans \tV\tV\trans}^2 + \norm{\e_i\trans \tU\tU\trans}^2
+(\tV\tV\trans)_{jj}(\tU\tU\trans)_{ii} + (\tU\tV\trans)_{ij}(\tU\tV\trans - \M)_{ij}
\right)
\end{align*}
The third order term:
\begin{align*}
Q_3 - Q_4 =&
2\tr(\Delta_\U \tV\trans\Delta_\V\Delta_\U\trans)
+ 2\tr(\Delta_\U \Delta_\V\trans \Delta_\V\tU\trans) \nn \\
=&-16 \eta^3 d^6(\U\V\trans - \M)^3_{ij}\left(
(\tV\tV)_{jj}(\tU\tV)_{ij} + (\tU\tV)_{ij} (\tU\tU)_{ii} \right)
\end{align*}
The forth order term:
\begin{align*}
Q_4 = \tr(\Delta_\U \Delta_\V\trans\Delta_\V\Delta_\U\trans)
=16\eta^4d^8(\U\V\trans - \M)^4_{ij}(\tU\tV)_{ij}^2
\end{align*}

Again, in addition to equation \eqref{aux_1}, \eqref{aux_1}, \eqref{aux_2}, \eqref{aux_3}, we also need following inequality:
\begin{align}
&\norm{\tU\tV\trans - \M}_{\infty}1_{\fE_t} = \max_{ij} |\tr(\e_i\trans (\tU\tV\trans - \M) \e_j)|1_{\fE_t} \nn\\
= &\max_{ij} |\tr(\e_i\trans (\proj_{\X} + \proj_{\X_\perp}) (\tU\tV\trans - \M) \e_j)|1_{\fE_t} \nn\\
\le & \max_{ij} |\tr(\e_i\trans \proj_{\X}  (\tU\tV\trans - \M) \e_j)|1_{\fE_t}
+ \max_{ij} |\tr(\e_i\trans \proj_{\X_\perp} \tU\tV\trans \e_j)|1_{\fE_t} \nn \\
\le & \max_{i}\norm{\e_i\trans \X} \norm{\tU\tV\trans-\M}_F1_{\fE_t} + \max_{j}\norm{\e_j\trans \W_\V} \norm{\tU\tV\trans-\M}_F1_{\fE_t} \nn \\
\le & O(\cn\sqrt{\chi})\sqrt{f(\U_t)}
\end{align}
Now we are ready to prove Lemma.

\paragraph{For the first inequality $\E[Q_2 1_{\fE_t} |\fF_t] \le \eta^2 O(\mu dk \cn^2)f(\tU, \tV)1_{\fE_t}$:}
\begin{align*}
\E [Q_2 1_{\fE_t}|\fF_t]
\le \E [|Q_2-Q_3| 1_{\fE_t}|\fF_t] + \E [|Q_3-Q_4| 1_{\fE_t}|\fF_t] + \E [|Q_4| 1_{\fE_t}|\fF_t]
\end{align*}
For each term, we can bound as:
\begin{align*}
\E [|Q_2-Q_3| 1_{\fE_t}|\fF_t] 
\le & \eta^2 O( d^2) \sum_{ij}(\tU\tV\trans - \M)^2_{ij}
O(\chi) = \eta^2 O(d^2\chi) f(\tU, \tV)\\
\E [|Q_3-Q_4| 1_{\fE_t}|\fF_t]
\le &
\eta^3 O(d^4) \sum_{ij}|\tU\tV\trans - \M|^3_{ij} O(\chi)  \nn \\
\le & \eta^3 O(d^4 \chi) \norm{\tU\tV\trans - \M}_\infty f(\tU, \tV)   \le \eta^2 O(d^2\chi) f(\tU, \tV)  \\
\E [|Q_4| 1_{\fE_t}|\fF_t] \le &
\eta^4 O(d^6) \sum_{ij}(\tU\tV\trans - \M)^4_{ij} O(\chi) \nn \\
\le & \eta^4 O(d^6\chi) \norm{\tU\tV\trans - \M}^2_\infty f(\tU, \tV) \le \eta^2 O(d^2\chi) f(\tU, \tV)
\end{align*}
This gives in sum that 
\begin{equation*}
\E [Q_2 1_{\fE_t}|\fF_t] \le \eta^2 O(d^2\chi) f(\tU, \tV)1_{\fE_{t}}  = \eta^2 O(\mu dk \cn^2) f(\tU, \tV)1_{\fE_{t}}
\end{equation*}

\paragraph{For the second inequality $|Q_1|1_{\fE_t} \le \eta O(\mu d k \cn^{3}) f(\tU, \tV)1_{\fE_{t}} \text{~w.p~} 1 $:}
\begin{align*}
|Q_1|1_{\fE_{t}} \le |Q_1-Q_2|1_{\fE_{t}}  + |Q_2- Q_3|1_{\fE_{t}}  + |Q_3-Q_4|1_{\fE_{t}}  + |Q_4|1_{\fE_{t}} 
\end{align*}
For each term, we can bound as:
\begin{align*}
|Q_1-Q_2|1_{\fE_{t}} \le &
\eta O(d^2)\norm{\tU\tV\trans - \M}_{\infty}\norm{\tU\tV\trans - \M}_{F} O(\sqrt{\chi})
\le \eta O(d^2\chi \cn)f(\tU, \tV) \nn \\
|Q_2- Q_3|1_{\fE_{t}} \le &\eta^2 O(d^4) \norm{\tU\tV\trans - \M}^2_{\infty}O(\chi)
\le \eta^2 O(d^4\chi^2 \cn^2) f(\tU, \tV) =\eta O(d^2\chi \cn)f(\tU, \tV) \nn \\
|Q_3- Q_4|1_{\fE_{t}} \le &\eta^3 O(d^6) \norm{\tU\tV\trans - \M}^3_{\infty}O(\chi)
\le \eta O(d^2\chi \cn)f(\tU, \tV) \nn \\
|Q_4|1_{\fE_{t}} \le &\eta^4 O(d^8) \norm{\tU\tV\trans - \M}^4_{\infty}O(\chi)
\le \eta O(d^2\chi \cn)f(\tU, \tV)
\end{align*}
This gives in sum that, with probability 1:
\begin{equation*}
|Q_1|1_{\fE_{t}}  \le \eta O(d^2\chi \cn)f(\tU, \tV)1_{\fE_{t}}   = \eta O(\mu d k \cn^{3})f(\tU, \tV)
1_{\fE_{t}} 
\end{equation*}

\paragraph{For the third inequality $\E [Q_1^2 1_{\fE_{t}}|\fF_{t}] \le \eta^2 O(\mu d k \cn^2) f^2(\tU, \tV)1_{\fE_{t}}$:}
\begin{equation*}
\E Q_1^2 1_{\fE_{t}}
\le 4 \left[\E (Q_1-Q_2)^2 1_{\fE_{t}} + \E (Q_2-Q_3)^2 1_{\fE_{t}} + \E (Q_3-Q_4)^2 1_{\fE_{t}} + \E Q_4^2 1_{\fE_{t}} \right]
\end{equation*}
For each term, we can bound as:
\begin{align*}
\E (Q_1-Q_2)^2 1_{\fE_{t}} \le & \eta^2 O(d^2) \sum_{ij} (\tU\tV\trans - \M)^2_{ij}
\norm{\tU\tV\trans - \M}^2_{F} O(\chi)
\le \eta^2 O(d^2\chi) f^2(\tU, \tV) \nn \\
\E (Q_2-Q_3)^2 1_{\fE_t} 
\le & \eta^4 O( d^6) \sum_{ij}(\tU\tV\trans - \M)^4_{ij}O(\chi^2)
 \nn \\
\le & \eta^4 O(d^6\chi^2)\norm{\tU\tV\trans - \M}_\infty^2\norm{\tU\tV\trans - \M}^2_{F}
\le \eta^4 O(d^6\chi^3 \cn^2)f^2(\tU, \tV) \nn \\
\le &\eta^2 O(d^2\chi) f^2(\tU, \tV)
\nn \\
\E (Q_3-Q_4)^2 1_{\fE_t}
\le &
\eta^6 O(d^{10}) \sum_{ij}|\tU\tV\trans - \M|^6_{ij} O(\chi^2)  \nn \\
\le & \eta^6 O(d^{10}\chi^2) \norm{\tU\tV\trans - \M}^4_\infty \norm{\tU\tV\trans - \M}^2_{F} \le \eta^2 O(d^2\chi) f^2(\tU, \tV)\\
\E Q_4^2 1_{\fE_t} \le &
\eta^8 O(d^{14}) \sum_{ij}(\tU\tV\trans - \M)^8_{ij} \delta_{il}O(\chi^2) \nn \\
\le & \eta^8 O(d^{14}\chi^2) \norm{\tU\tV\trans - \M}^6_\infty \norm{\tU\tV\trans - \M}^2_{F}
\le \eta^2 O(d^2\chi) f^2(\tU, \tV)
\end{align*}
This gives in sum that:
\begin{equation*}
\E Q_1^2 1_{\fE_{t}} \le \eta^2 O(d^2\chi) f^2(\tU, \tV)1_{\fE_{t}} = \eta^2 O(\mu d k \cn^2) f^2(\tU, \tV)1_{\fE_{t}}
\end{equation*}
This finishes the proof.
\end{proof}

\end{document}